\documentclass[11pt]{article}

\usepackage{ifthen}
\newif\ifshortver

\newcommand{\ifshort}[2]{\ifshortver#1\else#2\fi}

\title{TaSIL: Taylor Series Imitation Learning}
\usepackage{geometry}
\usepackage{hyperref}

\hypersetup{
    bookmarks=true,         % show bookmarks bar?
    unicode=false,          % non-Latin characters in Acrobat’s bookmarks
    pdftoolbar=true,        % show Acrobat’s toolbar?
    pdfmenubar=true,        % show Acrobat’s menu?
    pdffitwindow=false,     % window fit to page when opened
    pdfstartview={FitH},    % fits the width of the page to the window
    pdfnewwindow=true,      % links in new PDF window
    colorlinks=true,       % false: boxed links; true: colored links
    linkcolor=blue,          % color of internal links (change box color with linkbordercolor)
    citecolor=black,        % color of links to bibliography
    filecolor=magenta,      % color of file links
    urlcolor=blue,           % color of external links
    breaklinks=true
}
\usepackage{amsfonts}       % blackboard math symbols
\usepackage{nicefrac}       % compact symbols for 1/2, etc.
\usepackage{microtype}      % microtypography
\usepackage{xcolor}         % colors
\usepackage[font=small,labelfont=bf]{caption}               % captions
\usepackage{wrapfig}

\usepackage{Shorthands}
\usepackage[sort,compress]{natbib}

\usepackage{authblk}
\newcommand*\samethanks[1][\value{footnote}]{\footnotemark[#1]}
\author[3]{Daniel Pfrommer\thanks{Equal contribution}}
\author[1]{Thomas T.C.K.\ Zhang\samethanks}
\author[2]{Stephen Tu}
\author[1,2]{Nikolai Matni}
\affil[1]{Department of Electrical and Systems Engineering, University of Pennsylvania}
\affil[2]{Google Brain Robotics}
\affil[3]{Massachusetts Institute of Technology\thanks{This work was done while the author was affiliated with the University of Pennsylvania}}
\date{\vspace{-0.3cm}}

\begin{document}

\maketitle

\vspace{-1.2cm}

\begin{abstract}
We propose Taylor Series Imitation Learning (TaSIL), a simple augmentation to standard behavior cloning
losses in the context of continuous control.
TaSIL penalizes deviations in the higher-order Taylor series terms between
the learned and expert policies.  We show that experts satisfying a notion of \emph{incremental input-to-state stability} are easy to learn, in the sense that a small TaSIL-augmented imitation loss over expert trajectories guarantees a small imitation loss over trajectories generated by the learned policy.  We provide generalization
% sample-complexity 
bounds for TaSIL that scale as $\tilde\calO(1/n)$ in the realizable setting, for $n$ the number of expert demonstrations. Finally, we demonstrate experimentally the relationship between the robustness of the expert policy and the order of Taylor expansion required in TaSIL, and compare standard Behavior Cloning, DART, and DAgger with TaSIL-loss-augmented variants.  In all cases, we show significant improvement over baselines across a variety of MuJoCo tasks.
\end{abstract}

\section{Introduction}\label{sec: intro}

Imitation learning (IL), wherein expert demonstrations are used to train a policy \citep{hussein2017imitation,osa2018algorithmic}, has been successfully applied to a wide range of tasks, including self-driving cars \citep{pomerleau1989alvinn,codevilla2018end}, robotics \citep{schaal1999imitation}, and video game playing~\citep{ross2011reduction}.  While IL is typically more sample-efficient than reinforcement learning-based alternatives, it is also known to be sensitive to distribution shift: small errors in the learned policy can lead to compounding errors and ultimately, system failure \citep{pomerleau1989alvinn,ross2011reduction}.  In order to mitigate the effects of distribution shift caused by policy error, two broad approaches have been taken by the community.  On-policy approaches such as DAgger~\citep{ross2011reduction} augment the data-set with expert-labeled or corrected trajectories generated by learned policies. In contrast, off-policy approaches such as DART~\citep{laskey2017dart} and GAIL~\citep{ho2016generative} augment the data-set by perturbing the expert controlled system.  In both cases, the goal is to provide examples to the learned policy of how the expert recovers from errors.  While effective, these methods either require an interactive expert (on-policy), or access to a simulator for policy rollouts during training (off-policy), which may not always be practically feasible.  

In this work, we take a more direct approach towards mitigating distribution shift.  Rather than providing examples of how the expert policy recovers from error, we seek to endow a learned policy with the robustness properties of the expert directly.  In particular, we make explicit the underlying assumption in previous work that a good expert is able to recover from perturbations through the notion of \emph{incremental input-to-state stability}, a well-studied control theoretic notion of robust nonlinear stability. Under this assumption, we show that if the $p$-th order Taylor series approximation of the learned and expert policies approximately match on expert-generated trajectories, then the learned policy induces closed-loop behavior similar to that of the expert closed-loop system. Here the order $p$ is determined by the robustness properties of the expert, and makes quantitative the informal observation that more robust experts are easier to learn. Fundamentally, we seek to characterize settings where offline imitation learning is Probably Approximately Correctly (PAC)-learnable~\citep{shalev2014understanding}; that is, defining notions of expert data and designing algorithms such that low training time error on offline expert demonstrations implies low test time error along the learned policy's trajectories, in spite of distribution shift.

\ifshort{\vspace{-1em}}{}
\paragraph{Contributions}
We propose and analyze Taylor Series Imitation Learning (TaSIL), which augments the standard behavior cloning imitation loss to capture errors between the higher-order terms of the Taylor series expansion of the learned and expert policies.  We reduce the analysis of TaSIL to the analysis of a supervised learning problem over expert data: in particular, we identify a robustness criterion for the expert such that a small TaSIL-augmented imitation loss over expert trajectories guarantees that the difference between trajectories generated by the learned and expert policies is also small.  We also provide a finite-difference based approximation that is applicable to experts that cannot directly query their higher-order derivatives, expanding the practical applicability of TaSIL.  We show in the realizable setting that our algorithm achieves generalization bounds scaling as $\tilde\calO(1/n)$, for $n$ the number of expert demonstrations. 
Finally, we empirically demonstrate (i) that the relationship
between the robustness of the expert policy and the order of Taylor expansion
required in TaSIL predicted by our theory is observed in practice, and (ii) the benefits of using the TaSIL-augmented imitation loss by comparing the sample-efficiency of standard and TaSIL-loss-augmented behavior cloning, DART, and DAgger on a variety of MuJoCo tasks: on hard instances where behavior cloning fails, the TaSIL-augmented variants show significant performance gains. 

\ifshort{\subsection*{Related work}}{\subsection{Related work}}

\label{sec: related work}

\paragraph{Imitation learning} Behavior cloning is known to be sensitive to compounding errors induced by small mismatches between the learned and expert policies~\citep{pomerleau1989alvinn,ross2011reduction}. On-policy \citep{ross2011reduction} and off-policy \citep{laskey2017dart,ho2016generative} approaches exist that seek to prevent this distribution shift by augmenting the data-set created by the expert.  While DAgger is known to enjoy $\tilde\calO(T)$ sample-complexity in the task horizon $T$ for loss functions that are strongly convex in the policy parameters,\footnote{This bound degrades to $\tilde\calO(T^2)$ when loss function is only convex, and does not hold for the nonconvex loss functions we consider.} we are not aware of finite-data guarantees for DART or GAIL. 
In addition to these seminal papers, there is a body of work that seeks to leverage control theoretic techniques~\citep{hertneck2018learning,yin2020imitation} to ensure (robust) stability of the learned policy.  More closely related to our work are the results by~\citet{ren2021generalization} and~\citet{tu2021sample}.  In \citet{ren2021generalization}, a two-stage pipeline of imitation learning
followed by policy optimization via a PAC-Bayes regularization term is used to provide generalization bounds of the learned policy across randomly drawn environments. This work is mostly complementary to ours, as
TaSIL could in principle be used to augment the imitation losses used
in the first stage of their pipeline (leaving the second stage unmodified).

In~\citet{tu2021sample}, sample-complexity bounds for IL are provided under the assumption that the learned policy can be explicitly constrained to match the incremental stability properties of the expert.  While conceptually appealing, practically enforcing such stability constraints is difficult, and as such~\citet{tu2021sample} resort to heuristics in their implementation.  In contrast, we provide sample-complexity guarantees for a practically implementable algorithm, and under much milder stability assumptions.

\ifshort{\vspace{-1em}}{}
\paragraph{Robust stability and learning for continuous control} 
There is a rich body of work applying Lyapunov stability or contraction theory~\citep{lohmiller1998contraction} to learning for continuous control.   
For example~\citep{singh2020learning,lemme2014neural,ravichandar2017learning,sindhwani2018learning} use stability-based regularizers to trim the hypothesis space, and empirically show that this leads to more sample-efficient and robust learning methods.
Lyapunov stability and contraction theory have also been used to provide finite sample-complexity guarantees for adaptive nonlinear control~\citep{boffi2020regret} and learning stability certificates~\citep{boffi2020learning}.  
\ifshort{}{
\paragraph{Derivative matching in knowledge distillation}
When the expert and learned controller classes are neural networks, TaSIL can be viewed as a knowledge transfer/distillation algorithm~\citep{gou2021knowledge}. In this context, matching the learner's $\pder{\text{ output}}{\text{ input}}$ derivatives to the teacher's has been introduced under the labels of ``Jacobian Matching'' \citep{srinivas2018knowledge} and ``Sobolev Training'' \citep{czarnecki2017sobolev}. In fact, \cite{czarnecki2017sobolev} demonstrate experiments distilling RL agents, applied to Atari games. The idea of matching derivatives to improve sample-efficiency goes at least as far back as ``Explanation-Based Neural Network Learning'' (EBNN) introduced by \cite{mitchell1992explanation}. However, these works lack finite-sample guarantees and do not address the distribution shift issue, which are the key points of our work.

}

\newcommand{\state}{x}

\section{Problem formulation}
\label{sec: problem formulation}

% \paragraph{proposed notation changes}

% \begin{itemize}
%     \item Introduce system as $x_{t+1} = f(x_t, u_t)$ (no $v_t$).
%     \item Restrict our attention to studying its evolution under inputs of the form $u_t = \pi(x_t) + \Delta_t$, for $\pi$ a suitable policy, and $\Delta_t$ an additive input disturbance (to be used to capture policy errors).  Define closed-loop notation $x_{t+1} = f^\pi(x_t, \Delta_t) := f(x_t, \pi(x_t) + \Delta_t).$
%     \item Denote trajectories under $u_t = \pi(x_t) + \Delta_t$ by $\state_t^\pi(\xi, \{\Delta_s\}_{s=0}^{t-1})$.
%     \item Define a policy $\pi$ as locally $\delta$-ISS if $f^\pi(x,\Delta)$ is, i.e., if $\|\state_t^\pi(\xi, \{0\}_{s=0}^{t-1}) - \state_t^\pi(\xi, \{\Delta_s\}_{s=0}^{t-1})\|$ is well behaved relative to $\max_{k\leq t-1}\|\Delta_k\|.$
%     \item Since our goal policy is always $\pi_\star$, define $\Delta^{\pi_d}_{t}(\xi; \hat \pi)$ appropriately and loss $\ell^{\pi_d}(\xi; \hat \pi):=h(\{\Delta^{\pi_d}_{t}(\xi; \hat \pi)\}_t)$ for $\pi_d\in\{\pi_\star, \hat \pi\}$ the data generating distribution (being consistent with data-generating distribution as a superscript everywhere) 
% \end{itemize}

We consider the nonlinear discrete-time dynamical system
\begin{equation}\label{eq: DT system}
    % x_{t+1} = f(x_t, u_t+v_t, w_t),\;x_t \in \R^d, \;u_t,v_t \in \R^m, \;w_t \in \R^p
    x_{t+1} = f(x_t, u_t), \, x_0 = \xi,
\end{equation}
where $x_t \in \R^d$ is the system state,  $u_t \in \R^m$ is the control input, and $f: \R^d \times \R^m \to \R^d$ defines the system dynamics.  We study system \eqref{eq: DT system} evolving under the control input $u_t = \pi(x_t) + \Delta_t$, for $\pi:\R^d \to \R^m$ a suitable control policy, and $\Delta_t \in \R^m$ an additive input perturbation (that will be used to capture policy errors).  We define the corresponding \emph{perturbed} closed-loop dynamics by $x_{t+1} = \fcl{\pi}(x,\Delta_t):=f(x,\pi(x) + \Delta_t)$, and use $\state_t^\pi(\xi, \{\Delta_s\}_{s=0}^{t-1})$ to denote the value of the state $x_t$ at time $t$ evolving under control input $u_t = \pi(x) + \Delta_t$ starting from initial condition $x_0 = \xi$.  
%We denote a trajectory of length $T$ generated under the same conditions by $\Phi_T^\pi(\xi, \{\Delta_s\}_{s=0}^{t-1}) := \{\state_t^\pi(\xi, \{\Delta_s\}_{s=0}^{t-1})\}_{t=0}^{T-1}.$  
To lighten notation, we use $\state^\pi_t(\xi)$ to denote $\state_t^\pi(\xi, \{0\}_{s=0}^{t-1})$, and overload $\norm{\cdot}$ to denote the Euclidean norm for vectors and the operator norm for matrices and tensors.

Initial conditions $\xi$ are assumed to be sampled randomly from
a distribution $\calD$ with support restricted to a compact set $\calX$.  We assume access to $n$ rollouts of length $T$ from an expert $\pi_\star$, generated by drawing initial conditions $\{\xi_i\}_{i=1}^n$ i.i.d.\ from $\calD$. 
%This dataset is denoted $\curly{\Phi^{\pi_\star}_T(\xi_i)}_{i=1}^n$.
The IL task is to learn a policy $\hat{\pi}$ which leads to a closed-loop system with similar behavior to that induced by the expert policy $\pi_\star$ as measured by the \emph{expected imitation gap} $\Ex_\xi \max_{1\leq t \leq T} \norm{\state_t^{\hat{\pi}}(\xi) - \state_t^{\pi_\star}(\xi)}$.

The baseline approach to IL, typically referred to as \emph{behavior cloning} (BC), casts the problem as an instance of supervised learning.  Denote the discrepancy between an evaluation policy $\bar \pi$ and the expert policy $\pi_\star$ on a trajectory generated by a rollout policy $\pi_d$ starting at initial condition $\xi$ by $\Delta^{\pi_d}_t(\xi; \bar \pi) := \bar\pi(\state_t^{\pi_d}(\xi)) - \pi_\star(\state_t^{\pi_d}(\xi))$.  BC directly solves
the supervised empirical risk minimization (ERM) problem
\ifshort{\begin{align*}
    \hat{\pi}_{\mathsf{bc}} \in \argmin_{\pi \in \Pi} \tfrac{1}{n}\textstyle\sum_{i=1}^n h(\{ \Delta^{\pi_\star}_t(\xi_i; \pi) \}_{t=0}^{T-1}),
\end{align*}}
{\begin{align*}
    \hat{\pi}_{\mathsf{bc}} \in \argmin_{\pi \in \Pi} \frac{1}{n}\sum_{i=1}^n h(\{ \Delta^{\pi_\star}_t(\xi_i; \pi) \}_{t=0}^{T-1}),
\end{align*}}
over a suitable policy class $\Pi$.  Here, $h : (\R^m)^{T} \rightarrow \R$ is a loss function which encourages the discrepancy terms
$\Delta^{\pi_\star}_t(\xi_i; \pi)$ to be small
along the expert trajectories.
While behavior cloning is conceptually simple, it 
can perform poorly in practice due to 
distribution shifts triggered by errors in the learned policy $\hat{\pi}_{\mathsf{bc}}$. Specifically, due to the effects of compounding errors, the closed-loop system induced by the behavior cloning policy $\hat{\pi}_{\mathsf{bc}}$ may lead to a dramatically different distribution over system trajectories than that induced by the expert policy $\pi_\star$, even when the population risk on the expert data, $\Ex_\xi h( \{ \Delta^{\pi_\star}_t(\xi; \hat{\pi}_{\mathsf{bc}}) \}_{t=0}^{T-1} )$, is small.

As described in the introduction, existing approaches to mitigating distribution shift seek to augment the data with examples of the expert recovering from errors. These approaches either require an interactive oracle (e.g., DAgger) or access to a simulator for policy rollouts (e.g., DART, GAIL), and may not always be practically applicable.  To address the distribution shift challenge without resorting to data-augmentation, we propose TaSIL, an off-policy IL algorithm which provably leads to 
learned policies that are robust to distribution shift. 

The rest of the paper is organized as follows: in Section \ref{sec: stability and imitation gap}, we focus on ensuring robustness to policy errors for a single initial condition $\xi$.  We show that the imitation gap $\|\state_t^{\pi}(\xi)-\state_t^{\pi_\star}(\xi)\|$ between a test policy $\pi$ and the expert policy $\pi_\star$ can be controlled by the TaSIL-augmented imitation loss evaluated on the expert trajectory $\{\state_t^{\pi_\star}(\xi)\}$, effectively reducing the analysis of the imitation gap to a supervised learning problem over the expert data.  In Section \ref{sec: generalization bounds}, we integrate these results with tools from statistical learning theory to show that $n \gtrsim \varepsilon^{-r}/\delta$ trajectories are sufficient to achieve imitation gap of at most $\varepsilon$ with probability at least $1-\delta$; here $r>0$ is a constant determined by the stability properties of the expert policy $\pi_\star$, with more robust experts corresponding to smaller values of $r$. Finally, in Section \ref{sec: experiments}, we validate our analysis empirically, and show that using the TaSIL-augmented loss function in IL algorithms leads to significant gains in performance and sample efficiency.

\section{Bounding the imitation gap on a single trajectory}\label{sec: stability and imitation gap}

In this section, we fix an initial condition $\xi$ and test policy $\pi$, and seek to control the imitation gap $\Gamma_T(\xi;\pi):= \max_{1\leq t\leq T}\|\state_t^{\pi}(\xi, \{0\}) - \state_t^{\pi_\star}(\xi, \{0\})\|$.  A natural way to compare the closed-loop behavior of a test policy $\pi$ to that of the expert policy $\pi_\star$ is to view the discrepancy $\Delta^{\pi}_t(\xi;\pi):=\pi(\state^{\pi}_t(\xi)) - \pi_\star(\state^{\pi}_t(\xi))$ as an \emph{input perturbation} to the expert closed-loop system.  By writing $x_{t+1}=\fcl{\pi}(x_t,0)=\fcl{\pi_\star}(x_t, \Delta^{\pi}_t(\xi;\pi))$, the imitation gap can be written as $\Gamma_T(\xi; \pi) =  \max_t\|\state_t^{\pi_\star}(\xi, \{\Delta^{\pi}_s(\xi;\pi)\}_{s=0}^{t-1}) - \state_t^{\pi_\star}(\xi, \{0\})\|$, suggesting that closed-loop expert systems that are robust to input perturbations, as measured by the difference between nominal and perturbed trajectories, will lead to learned policies that enjoy smaller imitation gaps.

Stability conditions defined in terms of differences between nominal and perturbed trajectories have been extensively studied in robust nonlinear control theory via the notion of \emph{incremental input-to-state stability} ($\delta$-ISS) (see e.g.,~\citet{angeli2002lyapunov} and references therein).  Before proceeding, we recall definitions of standard comparison functions~\citep{khalil2002nonlinear}: a function $\gamma(x)$ is class $\mathcal{K}$ if it is continuous, strictly increasing, and satisfies $\gamma(0) = 0$, and a function $\beta(x,t)$ is class $\mathcal{KL}$ if it is continuous, $\beta(\cdot, t)$ is class $\mathcal{K}$ for each fixed $t$, and $\beta(x, \cdot)$ is decreasing for each fixed $x$.

\begin{definition}[$\delta$-ISS system]\label{def: delta ISS}
    Consider the closed-loop system evolving under policy $\pi$ and subject to perturbations $\Delta_t$ given by $x_{t+1} = \fcl{\pi}(x_t, \Delta_t)$. The closed-loop system $\fcl{\pi}(x_t, \Delta_t)$ is \emph{incremental input-to-state stable} ($\delta$-ISS) if there exists a class $\calK\calL$ function $\beta$ and a class $\calK$ function $\gamma$ such that for all initial conditions $\xi_1, \xi_2 \in \calX$, perturbation sequences $\curly{\Delta_t}_{t \geq 0}$, and $t\in\mathbb N$:
    \begin{align} \label{eq: delta ISS}
        \norm{\state_t^{\pi}(\xi_1; \{\Delta_s\}_{s=0}^{t-1}) - \state_t^{\pi}(\xi_2; \{0\}_{s=0}^{t-1})} &\leq \beta\paren{\norm{\xi_1 - \xi_2}, t} + \gamma\paren{\max_{0 \leq k \leq t - 1} \norm{\Delta_k}}.
    \end{align}
\end{definition}
Definition \ref{def: delta ISS} says that: (i) trajectories generated by $\delta$-ISS systems converge towards each other if they begin from different initial conditions, and (ii) the effect of bounded perturbations $\{\Delta_t\}$ on trajectories is bounded.  Our results will only require the stability conditions of Definition~\ref{def: delta ISS} to hold for a class of norm bounded perturbations.  In light of this, we say that a system is $\eta$-locally $\delta$-ISS if equation~\eqref{eq: delta ISS} holds for all input perturbations satisfying $\sup_{t \in \N}\|\Delta_t\|\leq \eta$.

By writing $x_{t+1}=\fcl{\pi}(x_t,0)=\fcl{\pi_\star}(x_t, \Delta^{\pi}_t(\xi;\pi))$, we conclude that if $x_{t+1} = \fcl{\pi_*}(x_t, \Delta_t)$ is $\eta$-locally $\delta$-ISS , and that $\sup_{t \in \N}\norm{\Delta^{\pi}_t(\xi;\hat\pi)} \leq \eta$, then
equation~\eqref{eq: delta ISS} yields
\begin{align}\label{eq: dISS and imitation gap}
    \norm{\state_t^\pi(\xi) - \state_t^{\pi_\star}(\xi)} &\leq \gamma\paren{ \max_{0 \leq k \leq t-1} \norm{\Delta^\pi_k(\xi; \pi)} } \Longrightarrow \Gamma_T(\xi;\pi) \leq \gamma\paren{\max_{0\leq t \leq T-1} \norm{\Delta_t^\pi(\xi;\pi)}}.
\end{align}
Equation~\eqref{eq: dISS and imitation gap} shows that the imitation gap $\norm{\state_t^\pi(\xi) - \state_t^{\pi_\star}(\xi)}$ is controlled by the maximum discrepancy $\|\Delta^\pi_t(\xi; \pi)\|$ incurred on trajectories generated by the policy $\pi$.  
A natural way of bounding the test discrepancy $\|\Delta^\pi_t(\xi; \pi)\|$, defined over trajectories generated by $\pi$, by the training discrepancy $\|\Delta^{\pi_\star}_t(\xi; \pi)\|$, defined over trajectories generated by $\pi_\star$, is to write out a Taylor series expansion of the former around the latter, i.e., to write:\footnote{We only take a first order expansion here for illustrative purposes.}
 \begin{align}
     \norm{\Delta^{\pi}_t(\xi; \pi)} 
 \leq \norm{\Delta^{\pi_\star}_t(\xi; \pi)} + \norm{\partial_x\Delta^{\pi_\star}_t(\xi;  \pi)}\Gamma_T(\xi;\pi) + \calO\left(\Gamma^2_T(\xi,\pi)\right), \label{eq: taylor series}
 \end{align}
 where $\partial_x\Delta^{\pi_\star}_t(\xi;  \pi)$ denotes the partial derivative of the discrepancy with respect to the argument of the policies $\pi$ and $\pi_\star$.
The challenge however, is that the imitation gap $\Gamma_T(\xi; \pi)$ also appears in the Taylor series expansion \eqref{eq: taylor series}, leading to an implicit constraint.  We show that this can be overcome if the order $p$ of the Taylor series expansion \eqref{eq: taylor series} is sufficiently large, as determined by the decay rate of the class $\calK$ function $\gamma(\cdot)$ defining the robustness of the expert policy.
We begin by identifying a condition reminiscent of adversarially robust training objectives that ensures a small imitation gap.
\begin{restatable}{proposition}{pertboundsinvgamma}
\label{prop: pert bounds inv-gamma}
Let the expert closed-loop system $\fcl{\pi_\star}$ be $\eta$-locally $\delta$-ISS for some $\eta > 0$. Fix an imitation gap bound $\varepsilon > 0$, initial condition $\xi$, and policy $\pi$. Then if 
\begin{align}
    \max_{0 \leq t \leq T-1} \sup_{\|\delta\| \leq \varepsilon} \|\pi_\star(\state_t^{\pi_\star}(\xi) + \delta) - \pi(\state_t^{\pi_\star}(\xi) + \delta)\| &\leq \min\{ \eta, \gamma^{-1}(\varepsilon)\}, \label{eqn:cl_ineq}
\end{align}
we have that the imitation gap satisfies  $\Gamma_T(\xi; \pi) \leq \varepsilon$.
\end{restatable}
Proposition~\ref{prop: pert bounds inv-gamma} states that if a policy $\pi$ is sufficiently close to the expert policy $\pi_\star$ in a tube around the expert trajectory, then the imitation gap remains small. How to ensure that inequality \eqref{eqn:cl_ineq} holds using only offline data from expert trajectories is not immediately obvious. The Taylor series expansion \eqref{eq: taylor series} suggests that a natural approach to satisfying this condition is to match derivatives of the test policy $\pi$ to those of the expert policy $\pi_\star$.  We show next that a sufficient order for such a Taylor series expansion is naturally determined by the decay rate of the class $\calK$ function $\gamma(x)$ towards $0$. Less robust experts have functions that decay to $0$ more slowly and will lead to stricter sufficient conditions. Conversely, more robust experts have functions that decay to $0$ more quickly, and will lead to more relaxed sufficient conditions.
We focus on two disjoint classes of class $\calK$ functions: (i) functions that decay in their argument faster than a linear function, i.e.\ $\gamma(x) < \calO(x)$ as $x \to 0^+$, and (ii) functions that decay in their argument no faster than a linear function, i.e.\ $\gamma(x) \geq \Omega(x)$ as $x \to 0^+$.

\ifshort{\vspace{-1em}}{}
% \paragraph{Rapidly decaying class $\calK$ functions}
\subsection*{Rapidly decaying class $\calK$ functions}
We show that when the class $\calK$ function $\gamma(x)$ decays to $0$ faster than $\calO(x)$ in some neighborhood of $0$, then matching the \emph{zeroth-order difference} $\max_t\norm{\Delta^{\pi_\star}_{t}(\xi; \pi)}$ on the expert trajectory, as is done in vanilla behavior cloning, suffices to close the imitation gap.  
We make the following assumption on the test policy $\pi$ and expert policy $\pi_\star$.
\begin{assumption}\label{assumption: lipschitz Pi}
    There exists a non-negative constant $L_\pi$ such that
        $\norm{\bar\pi(x) - \bar\pi(y)}_2 \leq L_\pi \norm{x - y}_2$ 
    for all $x$, $y$, and $\bar\pi \in \curly{\pi, \pi_\star}$.
\end{assumption}
Proposition~\ref{prop: pert bounds inv-gamma} then leads to the following guarantee on the imitation gap.

\begin{restatable}{theorem}{sublinearimitationgap}
\label{thm: sub-linear imitation gap}
   Fix a test policy $\pi$ and initial condition $\xi \in \calX$, and let Assumption~\ref{assumption: lipschitz Pi} hold.  Let $\fcl{\pi_\star}$ be $\eta$-locally $\delta$-ISS for some $\eta > 0$, and assume that the class $\calK$ function $\gamma(\cdot)$ in \eqref{eq: delta ISS} satisfies $\gamma(x) \leq \calO(x^{1+r})$ for some $r > 0$. Choose constants $\mu, \alpha > 0$ such that
    \begin{equation}\label{eq: sub-linear neighborhood relation}
    2L_\pi x + \paren{x/\mu}^{\frac{1}{1+r}} \leq \gamma^{-1}(x) \text{ for all $0 \leq x \leq \alpha$.}
    \end{equation}
    Provided that the imitation error on the expert trajectory incurred by $\pi$ satisfies: 
    \begin{align}
        &\max_{0 \leq t \leq T-1} \mu\norm{\Delta_t^{\pi_\star}(\xi; \pi)}^{1+r} \leq \alpha \label{eq: sub-linear closeness alpha}, \
        \max_{0 \leq t \leq T-1} 2L_\pi \mu \norm{\Delta_t^{\pi_\star}(\xi; \pi)}^{1+r} + \norm{\Delta_t^{\pi_\star}(\xi; \pi)} \leq \eta, 
    \end{align}
    then for all $1 \leq t \leq T$ the instantaneous imitation gap is bounded as 
    \begin{equation}\label{eq: sub-linear imitation gap}
        \|\state_t^{\pi_\star}(\xi) - \state_t^{\pi}(\xi)\| \leq \max_{0 \leq k \leq t-1} \mu\norm{\Delta_k^{\pi_\star}(\xi; \pi)}^{1+r}.
    \end{equation}
\end{restatable}
Theorem~\ref{thm: sub-linear imitation gap} shows that if a policy $\pi$ is a sufficiently good approximation of the expert policy $\pi_\star$ on an expert trajectory $\{\state_t^{\pi_\star}(\xi)\}$, then the imitation gap $\|\state_t^{\pi_\star}(\xi) - \state_t^{\pi}(\xi)\|$ can be upper bounded in terms of the discrepancy term $\max_{0\leq k \leq t-1}\|\Delta_t^{\pi_\star}(\xi; \pi)\|$ \emph{evaluated on the expert trajectory}.  To help illustrate the effect of the decay parameter $r$ on the choices of $\mu$ and $\alpha$, we make condition \eqref{eq: sub-linear neighborhood relation} more explicit by assuming that $\gamma(x)\leq Cx^{1+r}$ for all $x\in [0,1]$.  Then one can choose $\mu=2^{1+r}C$ and $\alpha=\calO(1)(L_{\pi}^{1+r} C)^{-1/r}.$  This makes clear that a larger $r$, i.e., a more robust expert, leads to less restrictive conditions \eqref{eq: sub-linear closeness alpha} 
on the policy $\pi$ and a tighter upper bound on the imitation gap \eqref{eq: sub-linear imitation gap} (assuming $\max_{0\leq k \leq t-1}\|\Delta_t^{\pi_\star}(\xi; \pi)\|<1$). In particular, for such systems, vanilla behavior cloning is sufficient to ensure bounded imitation gap. This also makes clear how the result breaks down when $r\approx0$, i.e.,\ when the decay rate is nearly linear, as the neighborhood $\alpha$ can become arbitrarily small, such that it may be impossible to learn a policy $\pi$ that satisfies the bound~\eqref{eq: sub-linear closeness alpha} with a practical number of samples $n$. 
The interplay of the bounds \eqref{eq: sub-linear closeness alpha} in Theorem~\ref{thm: sub-linear imitation gap} (and the subsequent theorems of its like) and the sample-complexity of imitation learning will be discussed in further detail in Section \ref{sec: generalization bounds}.

\ifshort{\vspace{-1em}}{}
% \paragraph{Slowly decaying class $\calK$ functions}
\subsection*{Slowly decaying class $\calK$ functions}
When the class $\calK$ function $\gamma(x)$ decays to $0$ slowly, for reasons discussed above, controlling the zeroth-order difference $\Delta^{\pi_\star}_t(\xi; \pi)$ may not be sufficient to bound the imitation gap. In particular, we consider class $\calK$ functions satisfying $\gamma(x) \leq \calO(x^{1/r})$ for some $r \geq 1$. Setting $p = \floor{r}$, we now show that matching up to the $p$-th total derivative of $\pi_\star$ is sufficient to control the imitation gap. Analogously to Assumption~\ref{assumption: lipschitz Pi}, we make the following regularity assumption on the test policy $\pi$ and expert policy $\pi_\star$.
\begin{assumption}\label{assumption: k-differentiable Pi}
    For a given non-negative $p \in \N$, assume that the test policy $\pi$ and expert policy $\pi_\star$ are $p$-times continuously differentiable, and there exists a constant $L_{\partial^{p} \pi} \geq 0$ such that
    \begin{align}
        \norm{\bar\pi(x) - \paren{J^{p}_{x_0} \bar\pi}(x)} \leq \frac{L_{\partial^p \pi}}{(p+1)!}\norm{x - x_0}^{p+1},
    \end{align}
    for all $x, x_0$ and $\bar\pi\in\curly{\pi, \pi_\star}$, where $\paren{J^{p}_{x_0}\bar\pi}(x) := \textstyle\sum_{j=0}^{p}\frac{1}{j!} \paren{\partial^j_x \bar\pi(x_0)}\paren{x - x_0}^{\otimes j}$ is the $p$-th order Taylor polynomial of $\bar\pi$ evaluated at $x_0$, and $\otimes$ denotes the tensor product.
\end{assumption}
With this assumption in hand, we provide the following guarantee on the imitation gap.

\begin{restatable}{theorem}{suplinearimitationgap}
    \label{thm: sup-linear imitation gap}
    Let $\fcl{\pi_\star}$ be $\eta$-locally $\delta$-ISS for some $\eta >0$, and assume that the class $\calK$ function $\gamma(\cdot)$ in \eqref{eq: delta ISS} satisfies $\gamma(x) \leq \calO(x^{1/r})$ for some $r\geq 1$.  Fix a test policy $\pi$ and initial condition $\xi \in \calX$, and let Assumption~\ref{assumption: k-differentiable Pi} hold for $p\in\mathbb{N}$ satisfying $p+1-r>0$. 
    Choose $\mu, \alpha > 0$ such that
    \begin{equation}\label{eq: sup-linear neighborhood relation}
    2\frac{L_{\partial^p \pi}}{(p+1)!} x^{p+1} + (x/\mu)^r \leq \gamma^{-1}(x), \text{   for all $0 \leq x \leq \alpha \leq \frac{1}{2}$}.
    \end{equation}
    Provided the $j$th total derivatives, $j=0,\dots,p$, of the imitation error on the expert trajectory incurred by $\pi$ satisfy:
    \begin{align}
        &\max_{0 \leq t \leq T-1} \max_{0 \leq j \leq p} \mu\paren{\frac{2}{j!}  \norm{\partial^j_x \Delta_t^{\pi_\star}(\xi; \pi)}}^{1/r} \leq \alpha  \label{eq: sup-linear closeness alpha}, \\
        &\max_{0 \leq t \leq T-1} \max_{0 \leq j \leq p} \frac{2 L_{\partial^p \pi}\mu^{p + 1}}{(p+1)!}\paren{\frac{2}{j!} \norm{\partial^j_x\Delta_t^{\pi_\star}(\xi; \pi)}}^{\frac{p+1}{r}} + \frac{2}{j!}  \norm{\partial^j_x \Delta_t^{\pi_\star}(\xi; \pi)} \leq \eta, \label{eq: sup-linear closeness eta}
    \end{align}
   then for all $1 \leq t \leq T$ the instantaneous imitation gap is bounded by 
    \begin{equation}\label{eq: super linear imitation gap}
        \|\state_t^{\pi_\star}(\xi) - \state_t^{\pi}(\xi)\| \leq \max_{0 \leq k \leq t-1} \max_{0 \leq j \leq p} \mu\paren{\frac{2}{j!}  \norm{\partial^j_x \Delta_t^{\pi_\star}(\xi; \pi)}}^{1/r}.
    \end{equation}
\end{restatable}
Theorem~\ref{thm: sup-linear imitation gap} shows that if the $p$-th order Taylor series of the policy $\pi$ approximately matches that of the expert policy $\pi_\star$ when evaluated on an expert trajectory $\{\state_t^{\pi_\star}(\xi)\}$, then the imitation gap $\norm{\state^{\pi_\star}_t(\xi) - \state^\pi_t(\xi)}$ can be upper bounded in terms of the derivatives of the discrepancy term, i.e., by $\max_{0\leq k \leq t-1} \max_{0\leq j \leq p} \norm{\partial^j_x\Delta_k^{\pi_\star}(\xi; \pi)}$, \emph{evaluated on the expert trajectory}. To help illustrate the effect of the choice of the order $p$ on the constants $\mu$ and $\alpha$, we make condition \eqref{eq: sup-linear neighborhood relation} more explicit by assuming that $\gamma(x)\leq Cx^{1/r}$ for all $x\in [0,1]$.  Then one can choose $\mu=2^{1/r}C$ and $\alpha=\calO(1)(L_{\partial^p\pi} C^r)^{-1/(p+1-r)}.$ This expression highlights a tradeoff: by picking larger order $p$, the right hand side $\alpha$ of bound \eqref{eq: sup-linear closeness alpha} increases, but at the expense of having to match higher-order derivatives.
This also highlights that both the order $p$ and closeness required by Equation~\eqref{eq: sup-linear closeness alpha} get increasingly restrictive as $r$ increases, matching our intuition that less robust experts lead to harder imitation learning problems. 

% \vspace{-1em}
\paragraph{Using estimated derivatives} We show in Appendix \ref{appendix: finite differencing} that the results of Theorems~\ref{thm: sub-linear imitation gap} and~\ref{thm: sup-linear imitation gap} extend gracefully to when only approximate derivatives $\widehat{\partial^j_x\pi_\star}(x)$ can be obtained, e.g., through finite-difference methods.  In particular, if $\|\widehat{\partial^j_x\pi_\star}(x)-\partial^j_x\pi_\star(x)\|\leq \varepsilon$, then it suffices to appropriately tighten the constraints \eqref{eq: sup-linear closeness alpha} and \eqref{eq: sup-linear closeness eta} by $\calO(\varepsilon^{1/r})$ and $\calO(\varepsilon)$, respectively.  Please refer to Appendix \ref{appendix: finite differencing} for more details.
\section{Algorithms and generalization bounds for TaSIL}\label{sec: generalization bounds}

The analysis of Section \ref{sec: stability and imitation gap} focused on a single test policy $\pi$ and initial condition $\xi$. Theorems \ref{thm: sub-linear imitation gap} and \ref{thm: sup-linear imitation gap} motivate defining the $p$-TaSIL loss function:
\ifshort{\begin{equation}\label{eq: TaSIL loss}
    \ell^{\pi_\star}_{p}(\xi; \pi):= \tfrac{1}{p+1}\textstyle\sum_{j=0}^{p} \max_{0\leq t \leq T-1}  \left\|\partial^j_x \Delta^{\pi_\star}_{t}(\xi;\pi) \right\|.
\end{equation}}
{\begin{equation}\label{eq: TaSIL loss}
    \ell^{\pi_\star}_{p}(\xi; \pi):= \frac{1}{p+1}\sum_{j=0}^{p} \max_{0\leq t \leq T-1}  \left\|\partial^j_x \Delta^{\pi_\star}_{t}(\xi;\pi) \right\|.
\end{equation}}
The corresponding policy $\hat{\pi}_{\mathsf{TaSIL},p}$ is the solution to the empirical risk minimization (ERM) problem:
\ifshort{\begin{align}\label{eq: TaSIL ERM}
    \hat{\pi}_{\mathsf{TaSIL},p} \in \argmin_{\pi \in \Pi} \tfrac{1}{n}\textstyle\sum_{i=1}^n\ell^{\pi_\star}_{p}(\xi_i; \pi),
\end{align}}
{\begin{align}\label{eq: TaSIL ERM}
    \hat{\pi}_{\mathsf{TaSIL},p} \in \argmin_{\pi \in \Pi} \frac{1}{n}\sum_{i=1}^n\ell^{\pi_\star}_{p}(\xi_i; \pi),
\end{align}}
which explicitly seeks to learn a policy $\pi\in \Pi$, for $\Pi$ a suitable policy class, that matches the $p$-th order Taylor series expansion of the expert policy.\footnote{Although we focus on the supremum loss $\max_{0\leq t \leq T-1} \left\|\partial^j_x \Delta^{\pi_\star}_{t}(\xi;\pi) \right\|$ in our analysis, we note that any surrogate loss that upper bounds the supremum loss, e.g., $\textstyle\sum_{t=0}^{T-1}\left\|\partial^j_x \Delta^{\pi_\star}_{t}(\xi;\pi) \right\|$, can be used.}  In this section, we analyze the generalization and sample-complexity properties of the $p$-TaSIL ERM problem \eqref{eq: TaSIL ERM}.
%
% performing a straightforward empirical risk minimization. From Theorems~\eqref{thm: sub-linear imitation gap} and \eqref{thm: sup-linear imitation gap}, we see that a natural risk measure to use is simply
% \begin{align}\label{eq: ERM objective}
%     h_j\left( \left\{ D^j \Delta_{t}(\xi_i; \pi)\right\}_{t=0,i=1}^{T-1,n} \right) := \frac{1}{n} \textstyle\sum_{i=1}^n \max_{0 \leq t \leq T - 1} \norm{D^j \Delta_{t}(\xi_i; \pi)}.
% \end{align}
%We want to show that if the empirical risk minimizer~\eqref{def: TaSIL ERM} incur small loss, then learned trajectories drawn from new initial conditions at test time with high probability also incur loss small enough to invoke our imitation gap guarantees. This requires an understanding of the generalization properties of the empirical risk minimizer. 
%
%To this end, we provide novel generalization guarantees for non-negative and bounded scalar loss classes via the machinery of local Rademacher complexities \citep{bartlett2005local}. 

Our analysis in this section focuses on the \emph{realizable} setting: we assume that for every dataset of expert trajectories $\{\{\state^{\pi_\star}_t(\xi_i)\}_{t=0}^{T-1}\}_{i=1}^n$, there exists a policy $\pi \in \Pi$ that achieves (near) zero empirical risk.  This is true if, for example, $\pi_\star \in \Pi$.
%Realizability captures many contemporary settings of interest where the policy class, such as feed-forward neural networks, is expressive enough to attain (near) zero empirical risk on any sample. 
In this setting, we demonstrate that we can attain generalization bounds that decay as $\tilde\calO(n^{-1})$, where $\tilde\calO(\cdot)$ hides poly-log dependencies on $n$. These rates are referred to as \emph{fast} rates in statistical learning,
since they decay faster than the $n^{-1/2}$ rate prescribed
by the central limit theorem.
 We present analysis for the non-realizable setting in Appendix \ref{appendix: generalization bounds proofs}: this analysis is standard and yields generalization bounds scaling as $\calO(n^{-1/2})$.  
%in essence, where usual tools can only guarantee the generalization error decreases as , we show the generalization error decreases as  \Thomas{remove last bit if we introduce $\tilde\calO$ earlier}.
%A particularly compelling property of our generalization bounds is that, in contrast to prior work (see \cite{srebro2010smoothness} and references therein), we do not require  smoothness or other regularity properties of the loss function beyond non-negativity and boundedness.  As such, we believe that they may be of independent interest.
%In fact, the objective we propose~\eqref{eq: ERM objective} is emphatically non-smooth. 
%Our implicit assumptions of non-negativity and boundedness are also generally not restrictive: non-negativity comes naturally from the norm-losses we consider, while boundedness can either be encoded into the policy class $\Pi$ or comes naturally from the underlying properties of the expert closed-loop system $\fcl{\pi_\star}$, such as stability in the sense of Lyapunov.

%We now introduce some notation. 
Let $\calG \subset [0,1]^\calX$ be a set of functions mapping some domain $\calX$ to $[0, 1]$.\footnote{This is without loss of generality for $[0, B]$-bounded functions by considering the normalized function class $B^{-1}\calG := \curly{B^{-1}g \, | \, g \in \calG} \subset [0, 1]^\calX$.} Let $\calD$ be a distribution with support restricted to $\calX$, and denote the mean of $g \in \calG$ with respect to $x \sim \calD$ by $\Ex_{x}[g]$. Similarly, fixing data points $x_1,\dots,x_n \in \calX$, we denote the empirical mean of $g$ by $\Ex_n[g] := n^{-1} \textstyle\sum_{i=1}^n g(x_i)$. We focus our analysis on the following class of parametric Lipschitz function classes.

\begin{definition}[Lipschitz parametric function class]
    A parametric function class $\calG \subset [0, 1]^\calX$
    is called $(B_\theta, L_\theta, q)$-\emph{Lipschitz}
    if $\calG = \{ g_\theta(\cdot) \mid \theta \in \Theta \}$ with $\Theta \subset \R^q$, 
    and it satisfies the following boundedness and 
    uniform Lipschitz conditions:
    \begin{align}\label{def: Lipschitz loss class}
        \sup_{\theta \in \Theta} \norm{\theta} \leq B_\theta, \:\:
        \sup_{x \in \calX} \sup_{\theta_1, \theta_2 \in \Theta, \theta_1 \neq \theta_2} \frac{\abs{g_{\theta_1}(x) - g_{\theta_2}(x)}}{\norm{\theta_1-\theta_2}} \leq L_\theta.
    \end{align}
    We assume without loss of generality that $B_\theta L_\theta \geq 1$.
\end{definition}
The description~\eqref{def: Lipschitz loss class} is very general, and as we show next, is compatible with feed-forward neural networks with differentiable activation functions. We then have the following generalization bound, which adapts \citep[Corollary 3.7]{bartlett2005local} to Lipschitz parametric function classes using the machinery of local Rademacher complexities \citep{bartlett2005local}. Alternatively, the result can
also be derived from \citep[Theorem 3]{haussler1992}.
\begin{restatable}{theorem}{localrademachergenbound}
\label{thm: local rademacher gen bound}
Let $\calG \subset [0, 1]^\calX$ be a $(B_\theta, L_\theta, q)$-Lipschitz parametric function class.
There exists a universal positive constant $K < 10^6$ such that the following holds. Given $\delta \in (0, 1)$,
with probability at least $1-\delta$ over the i.i.d.\ draws $x_1, \dots, x_n \sim \calD$, for all $g \in \calG$, the following bound holds:
% \Stephen{Are we sure that $B_\theta,L_\theta$ enter only in
% the log?}\Thomas{99\% sure; only place it could enter is Dudley, and I checked that part. I think it makes sense to me now.}
\begin{align}\label{eq: local rademacher gen bound}
    \Ex_x [g] \leq 2 \Ex_n [g] + K\left(\frac{q \log\paren{B_\theta L_\theta n} + \log(1/\delta)}{n} \right).
\end{align}
\end{restatable}
We now use Theorem~\ref{thm: local rademacher gen bound} to analyze the generalization properties of the $p$-TaSIL ERM problem \eqref{eq: TaSIL ERM}.
In what follows, we assume that the expert-closed loop system is stable in the sense of Lyapunov, i.e., that there exists $B_X>0$ such that $\sup_{t\in\N, \xi\in\calX}\norm{\state^{\pi_\star}_t(\xi)}\leq B_X$, and consider the following parametric class of $p+2$ continuously differentiable policies:
\begin{equation}\label{eq: Pi}
    \Pi_{\theta, p} := \curly{\pi(x, \theta) \mid \theta \in \R^q, \, \|\theta\|\leq B_\theta, \, \pi(0,\theta)=0 \; \forall \theta, \,  \text{$\pi$ is $p+2$ continuously differentiable}}.
\end{equation}
\ifshort{% SHORT VERSION %%%%%%%%%%%%%%%%%%%%%%%%%%%%%%%%%%%%%%%%%%%
Define $B_{j} := \sup_{\norm{x}\leq B_X, \norm{\theta}\leq B_\theta}\norm{\partial^j_x\pi(x, \theta)},$ $L_{j} := \sup_{\norm{x}\leq B_X, \norm{\theta}\leq B_\theta}\norm{\partial^{j+1}_x\partial_\theta\pi(x, \theta)}$ for $j=0,\dots,p$, and note that they are guaranteed to be finite under our regularity assumptions.  Finally, define the loss function class:
\begin{equation}\label{def: p-diff param loss class}
    \ell^{\pi_\star}_{p}\circ\Pi_{\theta, p} := \curly{\ell^{\pi_\star}_{p}(\cdot; \pi) \text{ defined in~\eqref{eq: TaSIL loss}} \mid  \pi \in \Pi_{\theta,p}}.
\end{equation}}
{% LONG VERSION %%%%%%%%%%%%%%%%%%%%%%%%%%%%%%%%%%%%%%%%%%%%%%%%%%%
Define the constants
\begin{align*}
    B_{j} := \sup_{\norm{x}\leq B_X, \norm{\theta}\leq B_\theta}\norm{\partial^j_x\pi(x, \theta)}, \:\: L_{j} := \sup_{\norm{x}\leq B_X, \norm{\theta}\leq B_\theta}\norm{\partial^{j+1}_x\partial_\theta\pi(x, \theta)},
\end{align*} 
for $j=0,\dots,p$, and note that they are guaranteed to be finite under our regularity assumptions.  Finally, define the loss function class:
\begin{equation}\label{def: p-diff param loss class}
    \ell^{\pi_\star}_{p}\circ\Pi_{\theta, p} := \curly{\ell^{\pi_\star}_{p}(\cdot; \pi) \text{ defined in~\eqref{eq: TaSIL loss}} \mid  \pi \in \Pi_{\theta,p}}.
\end{equation}}
%where we recall that $\lambda_j\geq 0, \, \textstyle\sum_{j=0}^{p}\lambda_j \leq 1$ are fixed regularization parameters.
From a repeated application of Taylor's theorem, we show in Lemma \ref{lem: loss params} that $B_{\ell,p}^{-1}(\ell^{\pi_\star}_{p}\circ\Pi_{\theta, p})$ is a $(B_\theta, B_{\ell,p}^{-1}L_{\ell,p}, q)$-Lipschitz parametric function class for $B_{\ell,p}:= \frac{2}{p+1}\textstyle\sum_{j=0}^{p}B_{j}$ and $L_{\ell,p}:=\frac{B_X}{p+1}\textstyle\sum_{j=0}^{p}L_{j}$. We now combine this with Theorem \ref{thm: local rademacher gen bound} to bound the population risk achieved by the solution to the TaSIL ERM problem \eqref{eq: TaSIL ERM}.

\begin{restatable}{corollary}{tasilgeneralizationbound}
\label{cor: TaSIL generalization bound}
    Let the policy class $\Pi_{\theta,p}$ be defined as in \eqref{eq: Pi}, and assume that $\pi_\star \in \Pi_{\theta,p}$.  Let the function class $\ell_{p}^{\pi_\star} \circ \Pi_{\theta, p}$ be defined as in~\eqref{def: p-diff param loss class}, and constants $B_{\ell, p}$, $L_{\ell, p}$ be defined as above.
    % in Lemma~\ref{lem: loss params}. %, such that $B_{\ell,p}^{-1}(\ell^{\pi_\star}_{\lambda, p}\circ\Pi_{\theta, p})$ is a $(B_\theta, L_{\ell,p}, q)$-Lipschitz parametric function class.
    % \begin{align*}
    %     g(\xi; \pi) &:= \frac{1}{B_{p,\pi}} \textstyle\sum_{j=0}^{p} \lambda_j \max_{0\leq t \leq T-1} \norm{D^j \Delta^{\pi_\star}_{t}(\xi;\pi) },
    % \end{align*}
    Let $\hat{\pi}_{\mathsf{TaSIL},p}$ be any empirical risk minimizer~\eqref{eq: TaSIL ERM}.
    %, and assume that we are in the realizable regime, such that $\Ex_n[\ell(\cdot, \hat{\pi}_{\mathsf{TaSIL},p}, \lambda)] = 0$ for any $\curly{\xi_i}_{i=1}^n \subset \calX$.
    Then with probability at least $1 - \delta$ over the initial conditions $\curly{\xi_i}_{i=1}^n \overset{\mathrm{i.i.d.}}{\sim} \calD^n$,
    \begin{align}\label{eq: TaSIL generalization bound}
        % \Ex\brac{\textstyle\sum_{j=0}^{p} \max_{0\leq t \leq T-1} \norm{D^j \Delta^{\pi_\star}_{t}(\xi;\hat{\pi}_{\mathsf{TaSIL},p}) }} 
        \Ex_{\xi} \brac{\ell^{\pi_\star}_{p}(\xi; \hat{\pi}_{\mathsf{TaSIL},p})} \leq \calO(1) B_{\ell, p}\frac{ q \log\paren{B_\theta B_{\ell, p}^{-1} L_{\ell, p} n} + \log(1/\delta)}{n}.
    \end{align}
\end{restatable}
% Instantiating this for our imitation learning setting, let $\calX$ be the set of initial conditions, and the functions $g \in \calG$
% \begin{align*}
%     g(\xi; \pi) &:= \frac{1}{B_{p,\pi}} \textstyle\sum_{j=0}^{p} \max_{0\leq t \leq T-1} \norm{D^j \Delta^{\pi_\star}_{t}(\xi;\pi) },
% \end{align*}
% where $B_{p,\pi} := \sup_{\pi \in \Pi}\sup_{\xi \in \calX} g(\xi; \pi)$ is a uniform upper bound on the loss a policy $\pi \in \Pi$ can incur on an expert trajectory, and $\Pi$ is a $(B_\theta, L_\theta, q)$-Lipschitz loss class. 
%As previously noted, the existence of $B_{\ell, p}$ is ensured through a variety of ways. For example, if $\fcl{\pi_\star}$ is stable in the sense of Lyapunov, and the set of initial conditions $\calX$ is almost-surely compact, then $B_{\ell, p}$ can be explicitly computed through \textit{a priori} knowledge of the system and the loss class. 
We note that since these generalization bounds solely concern the supervised learning problem of matching the expert on the expert trajectories, the constants do not depend on the stability properties of the trajectories generated by the learned policy $\hat{\pi}_{\mathsf{TaSIL},p}$. 
% Furthermore, Assumptions~\ref{assumption: lipschitz Pi} and~\ref{assumption: k-differentiable Pi} ensure the loss class $\calG$ is Lipschitz in the parameters
% \Thomas{note to self: mention that lipschitz bounds on partials guaranteed by Assumption 3.1/3.2}
To convert the generalization bound \eqref{eq: TaSIL generalization bound} to a probabilistic bound on the 
% expected
imitation gap $\Gamma_T(\xi; \hat{\pi}_{\mathsf{TaSIL},p})$, we first apply Markov's inequality to bound the probability that the conditions of Theorem \ref{thm: sub-linear imitation gap} or \ref{thm: sup-linear imitation gap} hold by the expected TaSIL loss \eqref{eq: TaSIL generalization bound}, and then apply Corollary~\ref{cor: TaSIL generalization bound} together with Markov's inequality.  %We instantiate this result for both rapidly and slowly decaying class $\calK$ functions.  %To do so,   Relating this to our imitation gap bounds, we may now use the bound in~\eqref{eq: TaSIL generalization bound} to derive a probabilistic guarantee on a trajectory sampled from a new initial condition and rolled out from $\hat{\pi}_{\mathsf{TaSIL},p}$ incur a small imitation gap.
\ifshort{% SHORT VERSION %%%%%%%%%%%%%%%%%%%%%%%%%%%%%%%%%%%%%%%%%%
\begin{theorem}[Rapidly decaying class $\calK$ functions]\label{thm: final fast}
    %Let %Assumption~\ref{assumption: lipschitz Pi} hold, $\fcl{\pi_\star}$ be $\eta$-locally $\delta$-ISS for some $\eta > 0$, and assume the class $\calK$ function $\gamma(\cdot)$ in~\eqref{def: delta ISS} satisfies $\gamma(x) \leq \calO(x^{1+r})$ for some $r > 0$. 
    Assume that $\pi_\star \in \Pi_{\theta,0}$ and let the assumptions of Theorem \ref{thm: sub-linear imitation gap} hold for all $\pi\in\Pi_{\theta,0}$. Let Equation~\eqref{eq: sub-linear neighborhood relation} hold with constants $\mu, \alpha > 0$, and assume without loss of generality that $\alpha/\mu \leq 1$, $L_\pi \mu \geq 1/2$. %
    %Let $B_{\ell,0}^{-1}(\ell^{\pi_\star}_{1, 0}\circ\Pi_{\theta, 0})$ be a $(B_\theta, L_{\ell,0}, q)$-Lipschitz parametric function class where
    % \begin{align*}
    %     \ell^{\pi_\star}_{e_0, 0}(\xi; \pi) = \max_{0\leq t \leq T-1}  \left\| \Delta^{\pi_\star}_{t}(\xi;\pi) \right\|,
    % \end{align*}
    % Define the $(B_\theta, L_\theta, q)$-Lipschitz parametric loss class $\calG = \curly{g_\theta \mid \theta \in \Theta \subseteq \mathbb{B}_2^q(B_\theta)}$,
    % \begin{align*}
    %     g_\theta(\xi) := g(\xi; \pi_\theta) = \frac{1}{B_{0, \pi}} \max_{0 \leq t \leq T - 1} \norm{\Delta^{\pi_\star}(\xi; \pi_\theta)}.    
    % \end{align*}
    Let $\hat{\pi}_{\mathsf{TaSIL},0}$ be an empirical risk minimizer of $\ell^{\pi_\star}_{0}$ over the policy class $\Pi_{\theta,0}$ for initial conditions $\curly{\xi_i} \overset{\mathrm{i.i.d.}}{\sim} \calD^n$. % and assume that we are in the realizable regime, such that $\Ex_n[g(\cdot, \hat{\pi}_{\mathsf{TaSIL},0})] = 0$.
    Fix a failure probability $\delta \in (0,1)$, and assume that 
    $n \geq \calO(1) \max\curly{B_{\ell, 0}  \tfrac{\kappa_\alpha}{\delta}\log\paren{\tfrac{\kappa_{\alpha}B_\theta L_{\ell, 0}}{\delta} }, \; B_{\ell, 0}  \tfrac{\kappa_\eta}{\delta} \log\paren{\tfrac{\kappa_\eta B_\theta  L_{\ell, 0}}{\delta} }},$
    where $\kappa_\alpha :=  q(\mu/\alpha)^{1/(1+r)}$, $\kappa_{\eta} := q L_\pi \mu/\eta$. Then with probability at least $1-\delta$, the imitation gap evaluated on $\xi \sim \calD$ (drawn independently from $\curly{\xi_i}_{i=1}^n$) satisfies
    % \begin{align*}
    %     n \geq O(1) \max\curly{B_{\ell, 0}  \frac{q(C/\alpha)^{\frac{1}{1 + r}} }{\delta} \log\paren{\frac{q (C/\alpha)^{\frac{1}{1 + r}} }{\delta}B_\theta L_\theta}, B_{\ell, 0}  \frac{q C L_\pi}{ \delta \eta} \log\paren{\frac{ q C L_\pi }{\delta \eta}B_\theta L_\theta}},
    % \end{align*}
   %
    % $\Gamma_T(\xi; \hat{\pi}_{\mathsf{TaSIL},0})$ satisfies
    % \Stephen{$\Gamma_T(\xi; \hat{\pi})$ is a random variable
    % (of both $\xi$ and $\hat{\pi}$).
    % Missing expectation?}
    % \Stephen{There is a notation clash-- there is $\delta$
    % for local ISS, and also $\delta$ for failure probability
    % of uniform concentration.} \Thomas{I'm doing union bound over the two events with $\delta/2$; wasn't sure separating the deltas was good for readability}
    %
        % \Gamma_T(\xi; \hat{\pi}_{\mathsf{TaSIL},0}) \leq O(1)\left(\frac{C}{\delta} \frac{B_{0,\pi} q \log\paren{\frac{B_\theta L_\theta}{B_{\ell, 0} } n } + B_{\ell, 0}\log(1/\delta)}{n}\right)^{1+r}
       \[
       \Gamma_T(\xi; \hat{\pi}_{\mathsf{TaSIL},0}) \leq \calO(1) \mu \Bigg(\frac{1}{\delta} \frac{B_{\ell, 0} q \log(B_\theta B_{\ell, 0}^{-1} L_{\ell, 0} n )}{n}\Bigg)^{1+r}.
       \]

\end{theorem}}
{% LONG VERSION %%%%%%%%%%%%%%%%%%%%%%%%%%%%%%%%%%%%%%%%%%%%%%%%%%%
\begin{theorem}[Rapidly decaying class $\calK$ functions]\label{thm: final fast}
    Assume that $\pi_\star \in \Pi_{\theta,0}$ and let the assumptions of Theorem \ref{thm: sub-linear imitation gap} hold for all $\pi\in\Pi_{\theta,0}$. Let Equation~\eqref{eq: sub-linear neighborhood relation} hold with constants $\mu, \alpha > 0$, and assume without loss of generality that $\alpha/\mu \leq 1$, $L_\pi \mu \geq 1/2$.
    Let $\hat{\pi}_{\mathsf{TaSIL},0}$ be an empirical risk minimizer of $\ell^{\pi_\star}_{0}$ over the policy class $\Pi_{\theta,0}$ for initial conditions $\curly{\xi_i} \overset{\mathrm{i.i.d.}}{\sim} \calD^n$. % and assume that we are in the realizable regime, such that $\Ex_n[g(\cdot, \hat{\pi}_{\mathsf{TaSIL},0})] = 0$.
    Fix a failure probability $\delta \in (0,1)$, and assume that 
    \begin{align*}
    n \geq \calO(1) \max\curly{B_{\ell, 0}  \tfrac{\kappa_\alpha}{\delta}\log\paren{\tfrac{\kappa_{\alpha}B_\theta L_{\ell, 0}}{\delta} }, \; B_{\ell, 0}  \tfrac{\kappa_\eta}{\delta} \log\paren{\tfrac{\kappa_\eta B_\theta  L_{\ell, 0}}{\delta} }},
    \end{align*}
    where $\kappa_\alpha :=  q(\mu/\alpha)^{1/(1+r)}$, $\kappa_{\eta} := q L_\pi \mu/\eta$. Then with probability at least $1-\delta$, the imitation gap evaluated on $\xi \sim \calD$ (drawn independently from $\curly{\xi_i}_{i=1}^n$) satisfies
     \begin{align*}
     \Gamma_T(\xi; \hat{\pi}_{\mathsf{TaSIL},0}) \leq \calO(1) \; \mu \left(\frac{1}{\delta} \frac{B_{\ell, 0} q \log\paren{B_\theta B_{\ell, 0}^{-1} L_{\ell, 0} n }}{n}\right)^{1+r}.
     \end{align*}
    
\end{theorem}
}

\ifshort{% SHORT VERSION %%%%%%%%%%%%%%%%%%%%%%%%%%%%%%%%%%%%%%%%%%%
\begin{theorem}[Slowly decaying class $\calK$ functions]\label{thm: final slow}
    Assume that $\pi_\star \in \Pi_{\theta,p}$, and let the assumptions of Theorem~\ref{thm: sup-linear imitation gap} hold for all $\pi \in \Pi_{\theta,p}$. Let Equation~\eqref{eq: sup-linear neighborhood relation} hold with constants $\mu, \alpha > 0$, and without loss of generality let $\big(\frac{\alpha}{2\mu})^r p! \leq 1$. Let $\hat{\pi}_{\mathsf{TaSIL},p}$ be an empirical risk minimizer of $\ell^{\pi_\star}_{p}$ over the policy class $\Pi_{\theta,p}$ for initial conditions $\curly{\xi_i} \overset{\mathrm{i.i.d.}}{\sim} \calD^n$.
    Fix a failure probability $\delta \in (0,1)$, and assume
    $n \geq \calO(1) \max_{j \leq p} \max\curly{ B_j\frac{\kappa_{\alpha, j}}{\delta} \log\paren{\frac{\kappa_{\alpha, j} B_\theta B_X L_j }{\delta} }, \;
         B_j\frac{\kappa_{\eta, j}}{\delta } \log\paren{\frac{\kappa_{\eta, j} B_\theta B_X L_j}{\delta} }}$, where $\kappa_{\alpha, j} :=  \big(\frac{2\mu}{\alpha}\big)^r\frac{p q }{j!}$ and $\kappa_{\eta, j} := \big(\frac{L_{\partial^p \pi}}{p!}\big(\frac{\mu}{\paren{j!}^{1/r}}\big)^{p+1} + \frac{p}{j!}\big) \frac{q}{\eta \delta}$. Then with probability at least $1 - \delta$, the imitation gap evaluated on $\xi \sim \calD$ (drawn independently from $\curly{\xi_i}_{i=1}^n$) satisfies
        \[
        \Gamma_T(\xi; \hat{\pi}_{\mathsf{TaSIL},p}) \leq \calO(1)\mu \max_{j \leq p} \left(\frac{p}{j! \delta} \frac{B_{j} q \log\paren{B_\theta B_j^{-1} B_X L_j n }}{n}\right)^{1/r}.
        \]
\end{theorem}}
{% LONG VERSION %%%%%%%%%%%%%%%%%%%%%%%%%%%%%%%%%%%%%%%%%%%%%%%%%%%%%
\begin{theorem}[Slowly decaying class $\calK$ functions]\label{thm: final slow}
    Assume that $\pi_\star \in \Pi_{\theta,p}$, and let the assumptions of Theorem~\ref{thm: sup-linear imitation gap} hold for all $\pi \in \Pi_{\theta,p}$. Let Equation~\eqref{eq: sup-linear neighborhood relation} hold with constants $\mu, \alpha > 0$. Let $\hat{\pi}_{\mathsf{TaSIL},p}$ be an empirical risk minimizer of $\ell^{\pi_\star}_{p}$ over the policy class $\Pi_{\theta,p}$ for initial conditions $\curly{\xi_i} \overset{\mathrm{i.i.d.}}{\sim} \calD^n$.
    Fix a failure probability $\delta \in (0,1)$, and assume
    \begin{align*}
        n \geq \calO(1) \max_{j \leq p} \max\curly{ B_j\frac{\kappa_{\alpha, j}}{\delta} \log\paren{\frac{\kappa_{\alpha, j} B_\theta B_j^{-1} B_X L_j }{\delta} }, \;
         B_j\frac{\kappa_{\eta, j}}{\delta } \log\paren{\frac{\kappa_{\eta, j} B_\theta B_j^{-1} B_X L_j}{\delta} }},
    \end{align*}
    where $\kappa_{\alpha, j} := \paren{\frac{\mu}{\alpha}}^{r}\frac{p q }{j!}$ and $\kappa_{\eta, j} := \big(\frac{L_{\partial^p \pi}}{(p+1)!}\frac{\mu^{p + 1}}{\paren{j!}^{(p+1)/r}} + \frac{1}{j!}\big) \frac{pq}{\eta \delta}$. Then with probability at least $1 - \delta$, the imitation gap evaluated on $\xi \sim \calD$ (drawn independently from $\curly{\xi_i}_{i=1}^n$) satisfies
    \begin{align*}
        \Gamma_T(\xi; \hat{\pi}_{\mathsf{TaSIL},p}) \leq \calO(1)\;\mu \max_{j \leq p} \left(\frac{p}{j! \delta} \frac{B_{j} q \log\paren{B_\theta B_j^{-1} B_X L_j n }}{n}\right)^{1/r}.
    \end{align*}
\end{theorem}
}

 In the rapidly decaying setting, corresponding to more robust experts, Theorem~\ref{thm: final fast} states that $n\gtrsim  \varepsilon^{-\frac{1}{1+r}}/\delta$ expert trajectories are sufficient to ensure that the imitation gap $\Gamma_T(\xi; \hat\pi_{\mathsf{TaSIL},0})\lesssim \varepsilon$ with probability at least $1-\delta$.  Recall that more robust experts have larger values of $r>0$, leading to smaller sample-complexity bounds.  In contrast, to achieve the same guarantees on the imitation gap in the slowly decaying setting, Theorem~\ref{thm: final slow} states $n\gtrsim \varepsilon^{-r}/\delta$ expert trajectories are sufficient, where we recall that less robust experts have larger values of $r\geq 1$. These theorems quantitatively show how the robustness of an underlying expert affects the sample-complexity of IL, with more robust experts enjoying better dependence on $\varepsilon$ than less robust experts. We note that analogous dependencies on expert stability are reflected in the burn-in requirements, i.e., the number of expert trajectories required to ensure with high probability no catastrophic distribution shift occurs, of each theorem.
\section{Experiments}\label{sec: experiments}
We compare three standard imitation learning algorithms, Behavior Cloning, DAgger, and DART, to TaSIL-augmented loss versions.  In TaSIL-augmented algorithms, we replace the standard imitation loss function
$$\ell^{\pi_\star}_{\mathsf{IL}}(\{\xi_i\}_{i=1}^n; \pi) := \tfrac{1}{n}\textstyle\sum_{i=1}^n \textstyle\sum_{t=0}^{T-1} \|\Delta_t^{\pi_\star}(\xi_i; \pi)\|$$
with the $p$-TaSIL-augmented loss
$$\ell^{\pi_\star}_{\mathsf{TaSIL}, p}(\{\xi_i\}_{i=1}^n; \pi) := \tfrac{1}{n}\textstyle\sum_{i=1}^n \textstyle\sum_{t=1}^{T-1}\textstyle\sum_{j=0}^p\lambda_j\|\mathrm{vec}(\partial^{j}_x\Delta_t^{\pi_{*}}(\xi_i; \pi))\|,$$
where $\{\lambda_j\}_{j=0}^p$ are positive tunable regularization parameters. We use the Euclidean norm of the vectorized error in the derivative tensors as more optimizer-amenable surrogate to the operator norm.

\ifshort{% SHORT VERSION %%%%%%%%%%%%%%%%%%%%%%%%%%%%%%%%%%%%%%%%%%%%
\begin{wrapfigure}{r}{.55\columnwidth}
    \centering
\includegraphics[width=0.49\linewidth]{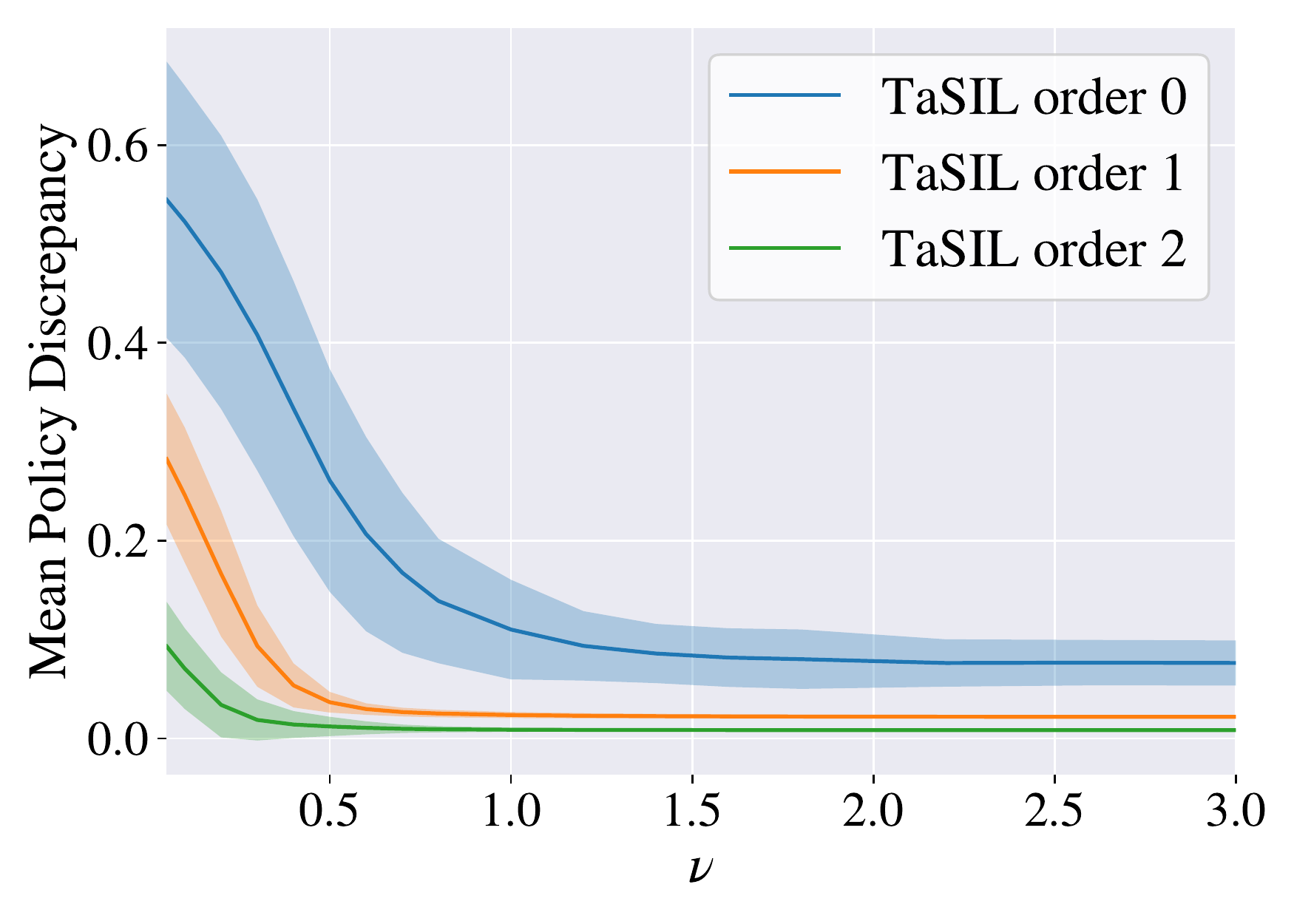}~\includegraphics[width=0.49\linewidth]{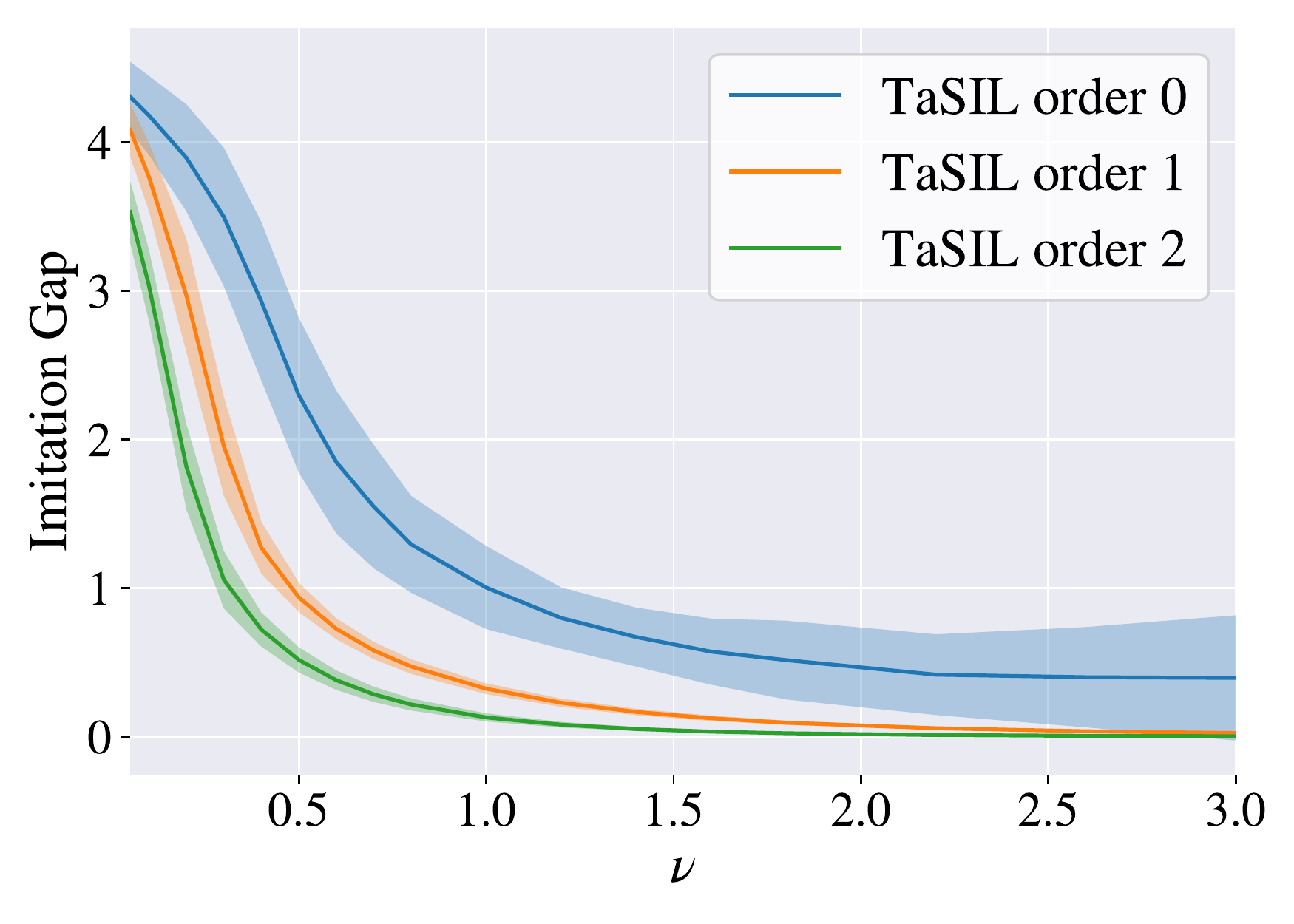}
    
    \caption{Left: The average Euclidean norm difference between the expert and learned policies on trajectories rolled out under the learned policy. Right: The maximum deviation between an expert and learned policy trajectory starting from identical initial conditions. All statistics are averaged across 50 test trajectories and we plot the mean and standard deviation for 10 random seeds.}
    \label{fig: gamma_sweep}
    %  \vspace{-12pt}
\end{wrapfigure}
}
{% LONG VERSION %%%%%%%%%%%%%%%%%%%%%%%%%%%%%%%%%%%%%%%%%%%%%%%%%%%%
\begin{figure}[t]
    \centering
\includegraphics[width=0.49\linewidth]{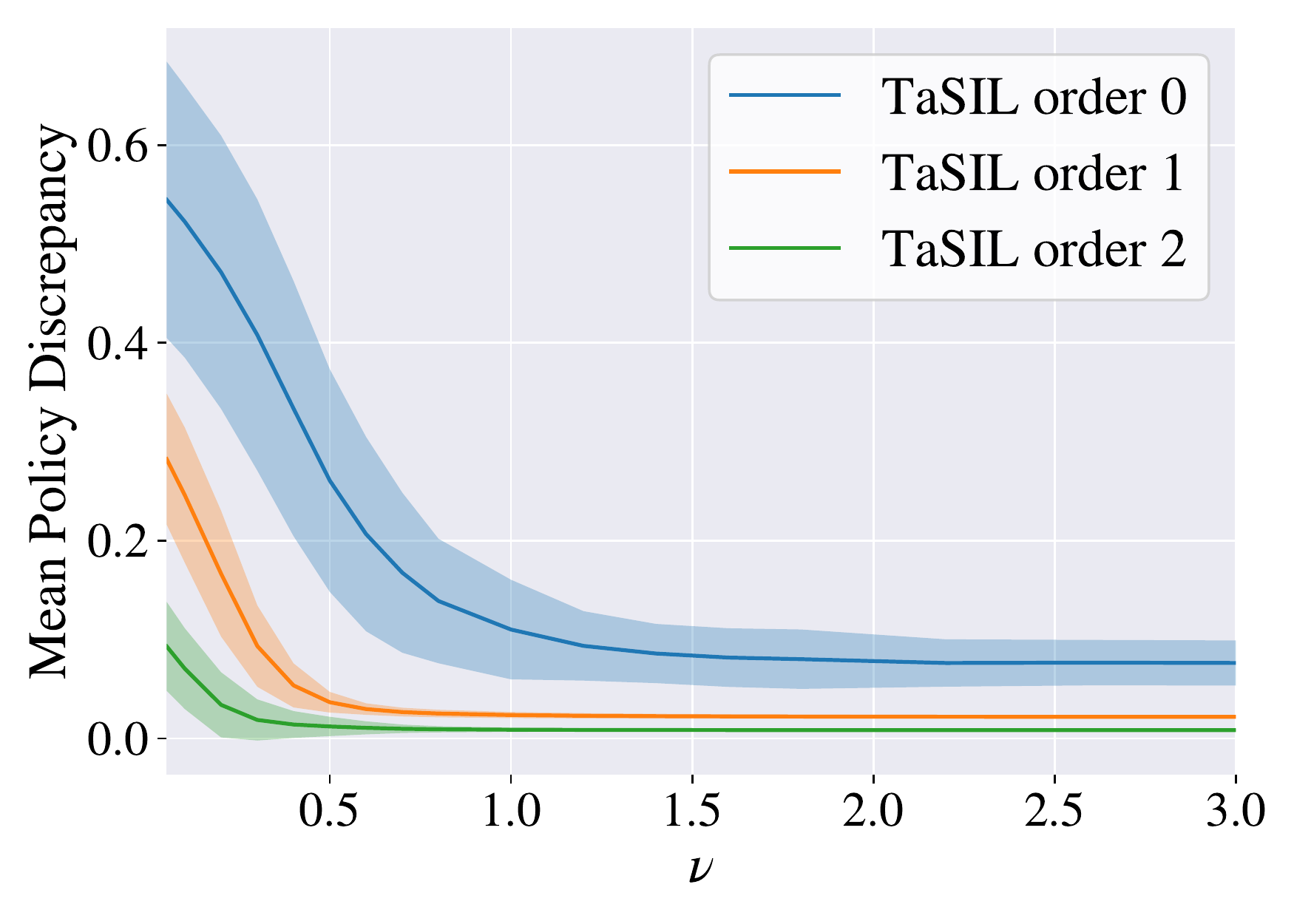}~\includegraphics[width=0.49\linewidth]{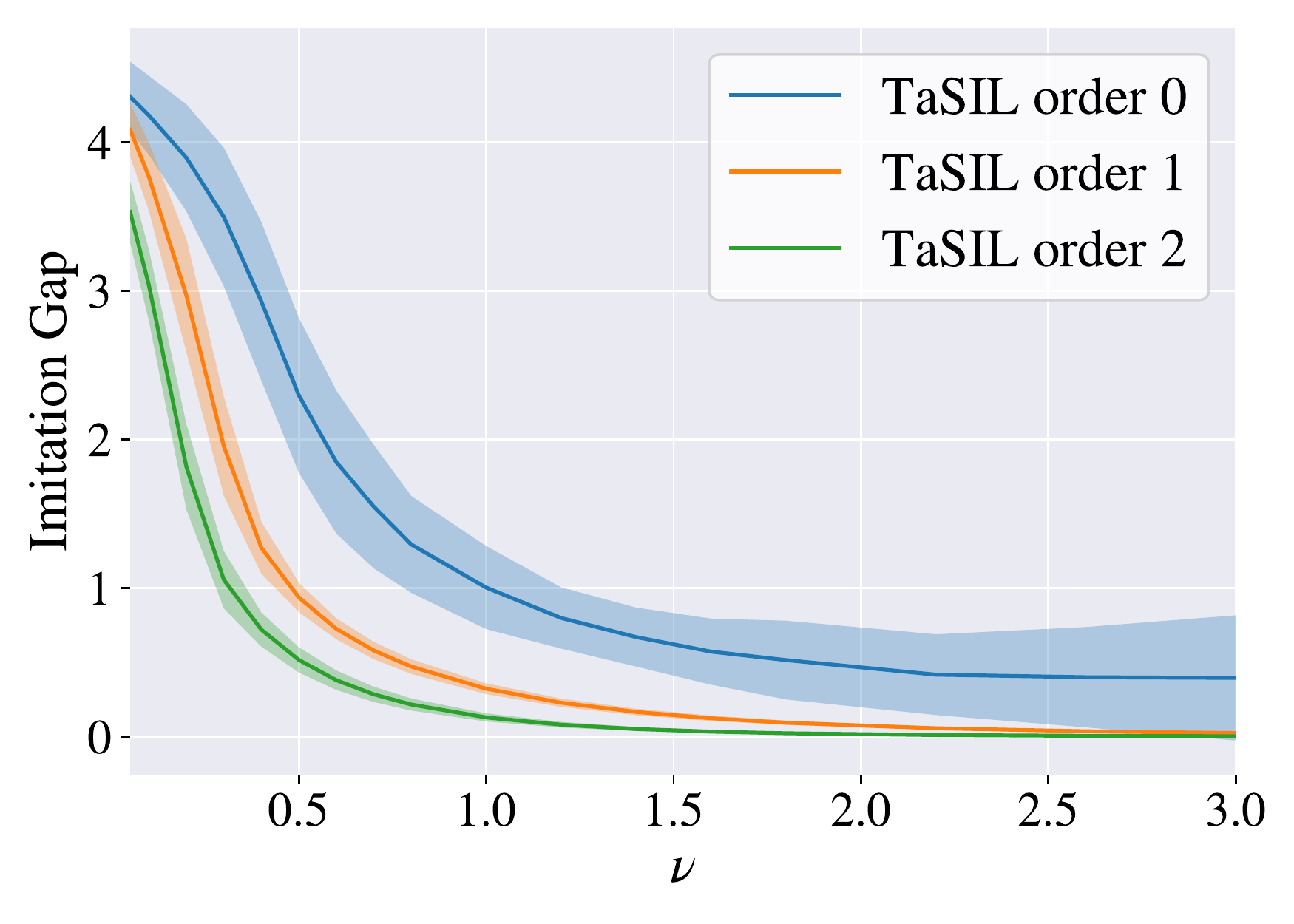}
    
    \caption{\textbf{Left:} The average Euclidean norm difference between the expert and learned policies on trajectories rolled out under the learned policy. \textbf{Right:} The maximum deviation between an expert and learned policy trajectory starting from identical initial conditions. All statistics are averaged across 50 test trajectories and we plot the mean and standard deviation for 10 random seeds.}
    \label{fig: gamma_sweep}
    %  \vspace{-12pt}
\end{figure}
}

We experimentally demonstrate (i) the effect of the expert stability properties and order of the TaSIL loss, and (ii) that TaSIL-loss-based imitation learning algorithms are more sample efficient than their standard counterparts.  All experiments\footnote{The code used for these experiments can be found at \url{https://github.com/unstable-zeros/TaSIL}} are carried out using Jax \citep{jax2018github} GPU acceleration and automatic differentiation capabilities to compute the higher-order derivatives, and the Flax \citep{flax2020github} neural network and Optax \citep{optax2020github} optimization toolkits. 

\ifshort{\vspace{-1em}}{}\paragraph{Stability Experiments} To illustrate the effect of the expert closed-loop system stability on sample-complexity, 
we consider a simple $\delta$-ISS stable dynamical system
with a tunable $\gamma$ input-to-state gain function. 
%we augment the $\delta$-ISS stable dynamical system presented by \citet{tu2021sample} with a tunable $\gamma$ input-to-state gain function. 
%
For state and input $x_t, u_t \in \R^{10}$, the dynamics are:
$$x_{t +1} = \eta x_t + (1 - \eta)\frac{\gamma(\|h(x_t) + u_t\|)}{\|h(x_t) + u_t\|} (h(x_t) + u_t).$$
 The perturbation function $h: \R^{10} \to \R^{10}$ is set to a randomly initialized MLP with two hidden layers of width 32 and GELU \citep{hendrycks2016gaussian} activations such that the expert $\pi_\star(x) = -h(x)$ yield a closed loop system $\fcl{\pi_\star}(x, \Delta)$ which is $\delta$-ISS stable with the specified class $\calK$ function $\gamma$ (see Appendix \ref{appendix: stability experiment details}). We use $\eta = 0.95$ for all experiments presented here.

We investigate the performance of $p$-TaSIL loss functions for $\delta$-ISS system with different class $\calK$ stability. We sweep $\calK$ functions $\gamma(x) = Cx^\nu$ for $\nu\in[0.05,3]$, $C = 5$ and $p$-TaSIL loss functions for $p \in  \{0, 1, 2\}$ (additional details can be found in Appendix \ref{appendix: stability experiment details}). The results of this sweep are shown in Figure \ref{fig: gamma_sweep}. Higher-order $p$-TaSIL losses significantly reduce both the imitation gap and the mean policy discrepancy on test trajectories. Notably, the first and second order TaSIL loss maintain their improved performance for slower decaying class $\calK$ functions. Theorem \ref{thm: sub-linear imitation gap} and Theorem \ref{thm: sup-linear imitation gap} yield lower bounds of $\nu = 1$, $\nu = 2^{-1}$, and $\nu =3^{-1}$ for closing the imitation gap using the $0$-TaSIL, $1$-TaSIL, and $2$-TaSIL losses respectively. Figure \ref{fig: gamma_sweep} demonstrates significant performance degradation in policy discrepancy and decaying imitation gap starting around these threshold values.

\vspace{-1em}\paragraph{MuJoCo Experiments} We evaluate the ability of the TaSIL loss to improve performance on standard imitation learning tasks by modifying Behavior Cloning, DAgger~\citep{ross2011reduction}, and DART~\citep{laskey2017dart} to use the $\ell_{\mathsf{TaSIL},1}$ loss and testing them in simulation on different OpenAI Gym MuJoCo tasks \citep{openai2016gym}. 
The MuJoCo environments we use and their corresponding (state, input) dimensions are: Walker2d-v3 ($17, 6$), HalfCheetah-v3 ($17$, $6$), Humanoid-v3 ($376,17$), and Ant-v3 ($111, 8$).

\ifshort{% SHORT VERSION %%%%%%%%%%%%%%%%%%%%%%%%%%%%%%%%%%%%%%%%%%
\begin{wrapfigure}{r}{.75\linewidth}
    \centering
    \includegraphics[width=1.0\linewidth]{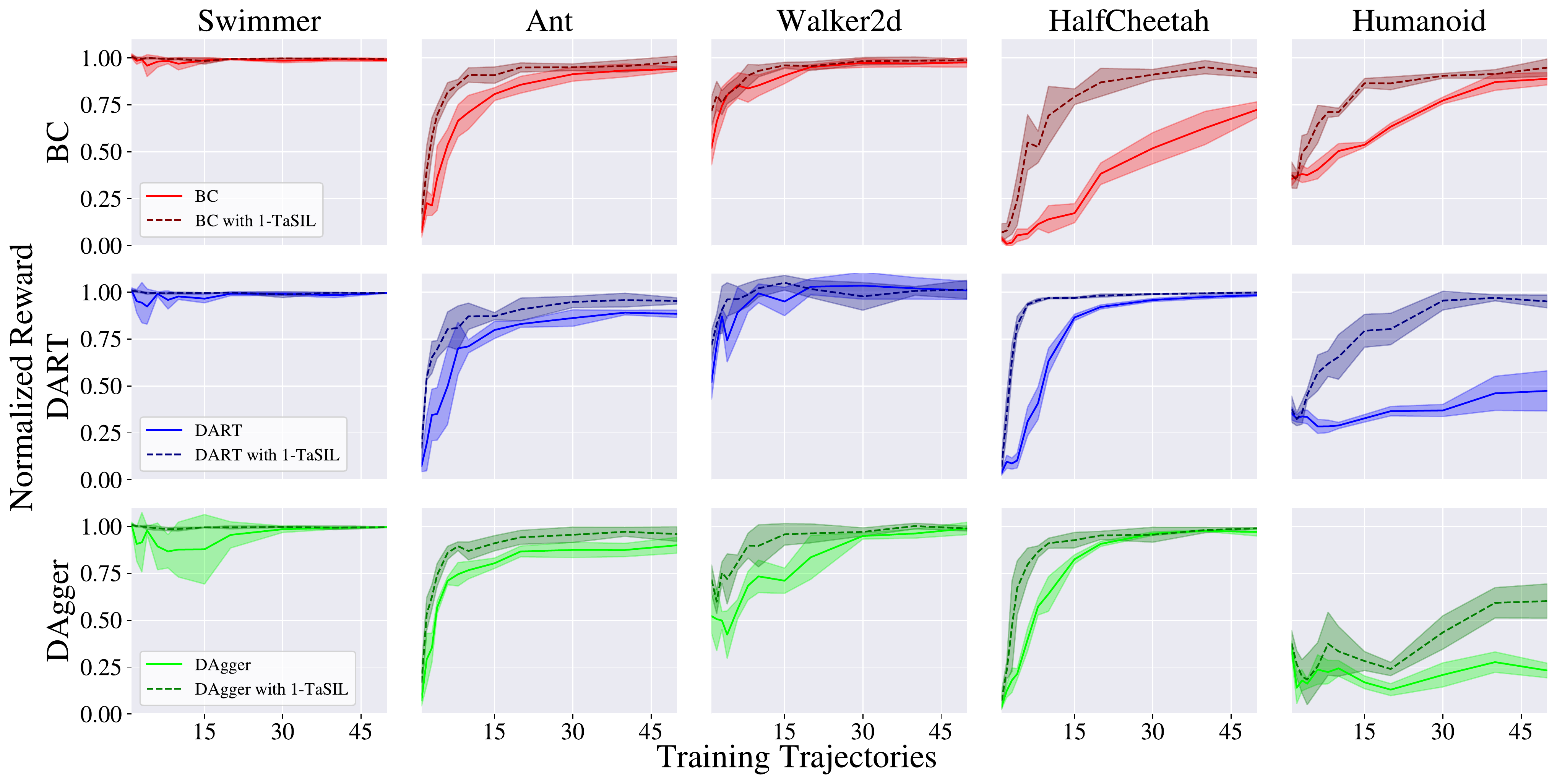}
    \caption{Cumulative expert-normalized rewards as a function of trajectory budget for policies trained using different algorithms with and without 1-TaSIL loss.
    }
    \label{fig: mujoco}
\end{wrapfigure}
}
{% LONG VERSION %%%%%%%%%%%%%%%%%%%%%%%%%%%%%%%%%%%%%%%%%%%%%%%%%%%
\begin{figure}[t]
    \centering
    \includegraphics[width=1.0\linewidth]{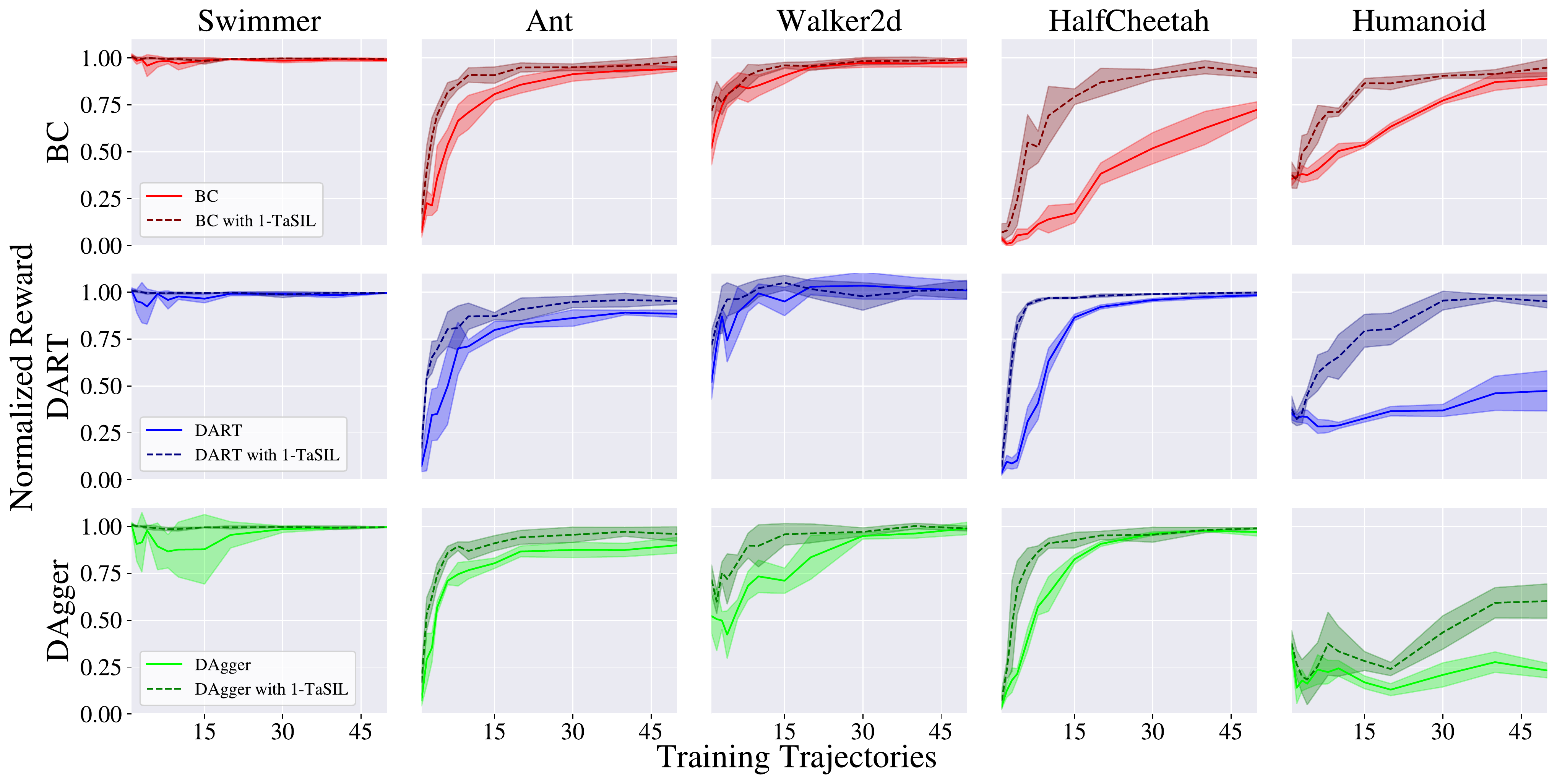}
    \caption{Cumulative expert-normalized rewards as a function of trajectory budget for policies trained using different algorithms with and without 1-TaSIL loss.
    }
    \label{fig: mujoco}
\end{figure}
}

For all environments we use pretrained expert policies obtained using Soft Actor Critic reinforcement learning by the Stable-Baselines3 \citep{stable-baselines3} project. The experts consist of Multi-Layer Perceptrons with two hidden layers of 256 units each and ReLU activations. For all environments, learned policies have 2 hidden layers with 512 units each and GELU activations in addition to Batch Normalization. The final policy output for both the expert and learned policy are rescaled to the valid action space after applying a tanh nonlinearity. We used trajectories of length $T = 300$ for all experiments. We refer to Appendix \ref{appendix: mujoco experiment details} for additional experiment details.

In Figure \ref{fig: mujoco} we report the mean expert-normalized rewards across $5$ seeds for all algorithms and environments as a function of the trajectory budget provided. Algorithms with 1-TaSIL loss showed significant improvement in sample-complexity across all challenging environments. The expert for the Swimmer environment is very robust due to the simplicity of the task, and so as predicted by Theorem~\ref{thm: sub-linear imitation gap} and~\ref{thm: final fast}, all algorithms (with the exception of the vanilla DAgger algorithm due to it initially selecting poor rollout trajectories) are able to achieve near expert performance across all trajectory budgets. We are also able to nearly match or exceed the performance of standard on-policy methods DAgger and DART with our off-policy 1-TaSIL loss Behavior Cloning in all environments.  Additional experimental results can be found in the appendix.  These include representative videos of the behavior achieved by the expert, BC, and TaSIL-augmented BC policies in the supplementary material, where once again, a striking improvement is observed, especially in low-data regimes for harder environments, as well as a systematic study of finite-difference-based approximations of $1$-TaSIL which achieve comparable performance to Jacobian-based implementations.

\vspace{-.5em}
\section{Conclusion} 
\vspace{-.5em}
We presented Taylor Series Imitation Learning (TaSIL), a simple augmentation to behavior cloning that penalizes deviations in the higher-order Tayler series terms between
the learned and expert policies.  We showed that $\delta$-ISS experts are easier to learn, both in terms of the loss-function that needs to be optimized and sample-complexity guarantees. Finally, we showed the benefit of using TaSIL-augmented losses in BC, DAgger, and DART across a variety of MuJoCo tasks.  This work opens up many exciting future directions, including extending TaSIL to pixel-based IL, and to (offline/inverse) reinforcement learning settings.

\section*{Acknowledgements}
We thank Vikas Sindhwani, Sumeet Singh, Jean-Jacques E.\ Slotine, Jake Varley, and Fengjun Yang for helpful feedback.
Nikolai Matni is supported by NSF awards CPS-2038873, CAREER award ECCS-2045834, and a Google Research Scholar award. 

\bibliographystyle{unsrtnat}
\bibliography{references}

\newpage
\appendix

\tableofcontents

\newpage

\section{Proofs for Section~\ref{sec: stability and imitation gap}}\label{appendix: imitation gap proofs}

\pertboundsinvgamma*
\begin{proof}
We do a proof by induction. 
\paragraph{Base case $t = 0$:} 
We trivially have at $t = 0$:
\[
\norm{\state_t^{\pi_\star}(\xi) - \state_t^{\pi}(\xi)} = \norm{\xi - \xi} = 0 \leq \epsilon.
\]

\paragraph{Induction step:} 
Assume for some $k > 0$, we have $\max_{t \leq k-1} \norm{\state_t^{\pi_\star}(\xi) - \state_t^{\pi}(\xi)} \leq \varepsilon$. We set $\delta := \state_{k-1}^{\pi_\star}(\xi) - \state_{k-1}^{\pi}(\xi)$ such that $\norm{\delta} \leq \varepsilon$. From Equation~\eqref{eqn:cl_ineq}, we are guaranteed that
\begin{align*}
\norm{\pi_\star(\state_{k-1}^{\pi}(\xi)) - \pi(\state_{k-1}^{\pi}(\xi))}    &= \norm{\pi_\star(\state_{k-1}^{\pi_\star}(\xi) + \delta) - \pi(\state_{k-1}^{\pi_\star}(\xi) + \delta)}\\
    &\leq \max_{0 \leq t \leq k-1} \sup_{\|\delta\| \leq \varepsilon} \|\pi_\star(\state_t^{\pi_\star}(\xi) + \delta) - \pi(\state_t^{\pi_\star}(\xi) + \delta)\| \\
    &\leq \min\curly{\eta, \gamma^{-1}(\varepsilon)}.
\end{align*}
Since $\fcl{\pi_\star}$ is $\eta$-locally $\delta$-ISS, we get from~\eqref{eq: dISS and imitation gap}
\begin{align*}
    \norm{\state_k^\pi(\xi) - \state_k^{\pi_\star}(\xi)} &\leq \gamma\paren{\max_{0 \leq s \leq k-1} \norm{\pi_{\star}\paren{\state_s^{\pi}(\xi)} - \pi\paren{\state_s^{\pi}(\xi)} }} \\
    &\leq \gamma\paren{\min\curly{\eta, \gamma^{-1}(\varepsilon)}} \\
    &\leq \varepsilon,
\end{align*}
and thus $\max_{t \leq k} \norm{\state_t^{\pi_\star}(\xi) - \state_t^{\pi}(\xi)} \leq \varepsilon$, completing the induction step.
\end{proof}

\sublinearimitationgap*
\begin{proof}
In order to leverage Proposition \ref{prop: pert bounds inv-gamma} we must first find a solution $\varepsilon$ to Equation \eqref{eqn:cl_ineq}. By Lipschitzness of the policy class, $$\max_{0 \leq t \leq T - 1} \sup_{\norm{\delta} \leq \varepsilon} \|\pi_\star(\state_t^{\pi_\star}(\xi) + \delta) - \pi(\state_t^{\pi_\star}(\xi) + \delta)\| \leq 2L_\pi \varepsilon + \max_{0 \leq t \leq T-1} \|\Delta_t^{\pi_\star}(\xi; \pi)\|,$$
and using the lower bound in Equation \eqref{eq: sub-linear neighborhood relation} it is therefore sufficient to find a solution $\varepsilon \leq \alpha$ to
\begin{align*}
    2L_\pi \varepsilon + \max_{0 \leq t \leq T - 1}  \|\Delta_t^{\pi_\star}(\xi;\pi)\| &\leq 2L_\pi \varepsilon + (\varepsilon/\mu)^{\frac{1}{1+ r}} \\
    \Longleftrightarrow \max_{0 \leq t \leq T - 1}  \|\Delta_t^{\pi_\star}(\xi;\pi)\| &\leq (\varepsilon/\mu)^{\frac{1}{1 + r}}.
\end{align*}
Picking $\varepsilon = \max_{0 \leq t \leq T-1} \mu\|\Delta_t^{\pi_\star}(\xi; \pi)\|^{1 + r}$ and adding the constraint $\varepsilon \leq \alpha$ in order to ensure the solution is sufficiently small allows use to apply Proposition \ref{prop: pert bounds inv-gamma} and obtain the final result
\begin{align*}
    \|\state_t^{\pi_\star}(\xi) - \state_t^{\pi}(\xi)\| \leq \max_{0 \leq t \leq T-1} \mu\|\Delta_t^{\pi_\star}(\xi; \pi)\|^{1 + r}.
\end{align*}
Provided that
\begin{align*}
    \max_{0 \leq t \leq T-1} \mu\|\Delta_t^{\pi_\star}(\xi; \pi)\|^{1 + r} \leq \alpha,\max_{0 \leq t \leq T-1}  2L_\pi \mu\|\Delta_t^{\pi_\star}(\xi; \pi)\|^{1 + r} + \|\Delta_t^{\pi_\star}(\xi; \pi)\| \leq \eta
\end{align*}
Thus completing the proof.
\end{proof}

\suplinearimitationgap*
\begin{proof} We proceed similarly as in the proof of Theorem~\ref{thm: sub-linear imitation gap}. From Proposition~\ref{prop: pert bounds inv-gamma}, we can take the $p$th Taylor expansion of the left hand side of Equation~\eqref{eqn:cl_ineq} and apply the triangle inequality a few times to yield:
\begin{align*}
        &\max_{0 \leq t \leq T-1} \sup_{\|\delta\| \leq \varepsilon} \|\pi_\star(\state_t^{\pi_\star}(\xi) + \delta) - \pi(\state_t^{\pi_\star}(\xi) + \delta)\|\\
    \leq\; &\max_{0 \leq t \leq T-1} \sup_{\|\delta\| \leq \varepsilon} \norm{\pi_\star(\state_t^{\pi_\star}(\xi)) - \pi(\state_t^{\pi_\star}(\xi))} \\
    &\qquad \qquad \qquad+ \norm{\pi_\star(\state_t^{\pi_\star}(\xi) + \delta) - \pi_\star(\state_t^{\pi_\star}(\xi)) - \paren{ \pi(\state_t^{\pi_\star}(\xi) + \delta) - \pi(\state_t^{\pi_\star}(\xi))}} \\
    \leq\; &\max_{0 \leq t \leq T-1} \sup_{\|\delta\| \leq \varepsilon} \norm{\pi_\star(\state_t^{\pi_\star}(\xi)) - \pi(\state_t^{\pi_\star}(\xi))} \\
    &\qquad \qquad \qquad+ \norm{\sum_{j=1}^p \frac{1}{j!}\partial_x^j \pi_\star\paren{\state^{\pi_\star}_t(\xi)} \cdot \delta^{\otimes j} - \sum_{j=1}^p \frac{1}{j!}\partial_x^j \pi\paren{\state^{\pi_\star}_t(\xi)} \cdot \delta^{\otimes j}} +2\frac{L_{\partial^p\pi}}{(p+1)!}\norm{\delta}^{p+1} \\
    \leq\; &\max_{0 \leq t \leq T-1} \sup_{\|\delta\| \leq \varepsilon} \sum_{j=0}^p \frac{1}{j!} \norm{\partial_x^j \Delta_t^{\pi_\star}(\xi; \pi) \cdot \delta^{\otimes j}} + 2\frac{L_{\partial^p \pi}}{(p+1)!} \norm{\delta}^{p+1} \\
    \leq\; &\max_{0 \leq t \leq T-1} \sum_{j=0}^p \frac{1}{j!} \norm{\partial_x^j \Delta_t^{\pi_\star}(\xi; \pi)}\varepsilon^j + 2\frac{L_{\partial^p \pi}}{(p+1)!} \varepsilon^{p+1}.
\end{align*}
Therefore, it suffices to find an $\varepsilon$ small enough such that
\[
\max_{0 \leq t \leq T-1} 2\frac{L_{\partial^p \pi}}{(p+1)!} \varepsilon^{p+1} +  \sum_{j=0}^p \frac{1}{j!}  \norm{\partial_x^j \Delta_t^{\pi_\star}(\xi; \pi)}\varepsilon^j \leq \gamma^{-1}(\varepsilon).
\]
Since we are given $\gamma(x) \leq \calO(x^{1/r})$, we have $\gamma^{-1}(x) \geq \Omega\paren{x^r}$. This motivates finding a large enough $\mu$ and small enough neighborhood $\alpha$ such that
\begin{align*}
    \max_{0 \leq t \leq T-1} 2\frac{L_{\partial^p \pi}}{(p+1)!} \varepsilon^{p+1} +  \paren{\frac{\varepsilon}{\mu}}^r \leq \gamma^{-1}(\varepsilon),
\end{align*}
for all $0 <  \varepsilon \leq \alpha \leq 1/2$. In essence, we want to find a sufficiently small neighborhood $\alpha$ such that the $\varepsilon^{p+1}$ term is dominated by the $\varepsilon^r$ term, while also selecting a $\mu$ such that the total sum is still upper bounded by $\gamma^{-1}(x) \geq \Omega(x^{r})$ in this neighborhood. The choice of raising $\varepsilon/\mu$ to the $r$-th power arises from the fact that $r$ is the smallest exponent--thus affecting the imitation gap in Equation~\eqref{eq: super linear imitation gap} downstream least severely--that ensures $\mu, \alpha$ will always exist. Having found such $\mu,\alpha$, we now simply have to find $\norm{\partial_x^j \Delta_t^{\pi_\star}(\xi; \pi)}$ small enough such that
\begin{align}
    \sum_{j=0}^p \frac{1}{j!}\norm{\partial_x^j \Delta_t^{\pi_\star}(\xi; \pi)}\varepsilon^j &\leq \max_{j \leq p} \frac{1}{j!}  \norm{\partial_x^j \Delta_t^{\pi_\star}(\xi; \pi)} \sum_{j = 0}^p \varepsilon^j \label{eqn: core_slow_ineq}\\
    &\leq \max_{j \leq p} \frac{2}{j!}  \norm{\partial_x^j \Delta_t^{\pi_\star}(\xi; \pi)} && \varepsilon \leq \alpha \leq 1/2 \nonumber \\
    &= \paren{\frac{\varepsilon}{\mu}}^r. \nonumber
\end{align}
Solving this for $\varepsilon$, we get
\[
    \varepsilon =  \max_{j \leq p} \mu\paren{\frac{2}{j!} \norm{\partial^j_x \Delta_t^{\pi_\star}(\xi; \pi)}}^{1/r},
\]
as long as $\varepsilon \leq \alpha$, the neighborhood condition, and $2\frac{L_{\partial^p \pi}}{(p+1)!} \varepsilon^{p+1} +  \paren{\frac{\varepsilon}{\mu}}^r \leq \eta$, the locality for $\delta$-ISS. These correspond to the conditions~\eqref{eq: sup-linear closeness alpha} and~\eqref{eq: sup-linear closeness eta}, respectively. This completes the 
proof.
\end{proof}

For completeness we present here a stronger variant of Theorem~\ref{thm: sup-linear imitation gap} for the special case where $p = r \in \N$. In this scenario we are able to remove the dependency of the imitation gap bounds on the $p$th order derivative provided it can be made sufficiently small.

\begin{theorem} \label{thm: slow decay exact r}
    Let $\fcl{\pi_\star}$ be $\eta$-locally $\delta$-ISS for some $\eta >0$, and assume that the class $\calK$ function $\gamma(\cdot)$ in \eqref{eq: delta ISS} satisfies $\gamma(x) \leq \calO(x^{1/r})$ for some $r\geq 1$.  Fix a test policy $\pi$ and initial condition $\xi \in \calX$, and let Assumption~\ref{assumption: k-differentiable Pi} hold with $p = r \in \mathbb{N}$. 
    Choose $\mu, \alpha > 0$ such that
    \begin{equation}
    2\frac{L_{\partial^p \pi}}{(p+1)!} x^{p+1} + (x/\mu)^p \leq \gamma^{-1}(x), \text{   for all $0 \leq x \leq \alpha \leq \frac{1}{2}$}.
    \end{equation}
    Provided the $j$th total derivatives, $j=0,\dots,p$, of the imitation error on the expert trajectory incurred by $\pi$ satisfy:
    \begin{align}
        \max_{0 \leq t \leq T-1} \max_{0 \leq j \leq p - 1} &\mu \left(\frac{4}{j!}  \norm{\partial^j_x \Delta_t^{\pi_\star}(\xi; \pi)}\right)^{1/p} \leq \alpha,  \label{eq: p=r lower derivatives}\\
        \max_{0 \leq t \leq T-1} \max_{0 \leq j \leq p - 1} &\frac{2 L_{\partial^p \pi}\mu^{p + 1}}{(p+1)!}\paren{\frac{4}{j!} \norm{\partial^j_x\Delta_t^{\pi_\star}(\xi; \pi)}}^{\frac{p+1}{p}} + \frac{4}{j!}  \norm{\partial^j_x \Delta_t^{\pi_\star}(\xi; \pi)} \leq \eta,  \\
        &\|\partial_x^p \Delta_t^{\pi_\star}(\xi; \pi)\| \leq \frac{p!}{2\mu^p} \label{eq: p=r p derivative}
    \end{align}
   then for all $1 \leq t \leq T$ the instantaneous imitation gap is bounded by 
    \begin{equation}
        \|\state_t^{\pi_\star}(\xi) - \state_t^{\pi}(\xi)\| \leq \max_{0 \leq k \leq t-1} \max_{0 \leq j \leq p - 1} \mu \paren{\frac{4}{j!}}^{1/r}  \norm{\partial^j_x \Delta_t^{\pi_\star}(\xi; \pi)}^{1/r}.
    \end{equation}
\end{theorem}
\begin{proof}
We follow the proof of Theorem~\ref{thm: sup-linear imitation gap} until Equation~\eqref{eqn: core_slow_ineq}. We then wish to solve
\begin{align*}
    \sum_{j=0}^{p} \frac{1}{j!} \|\partial_x^j \Delta_t^{\pi_\star}(\xi; \pi)\|\varepsilon^j \leq \left(\frac{\varepsilon}{\mu}\right)^{p}.
\intertext{Since the order of the RHS is $p$, provided that $\frac{1}{p!} \|\partial_x^p \Delta_t^{\pi_\star}(\xi; \pi)\| \leq \frac{1}{2}\frac{1}{\mu^p}$ we can write}
    \sum_{j=0}^{p - 1} \frac{1}{j!}\|\partial_x^j \Delta_t^{\pi_\star}(\xi; \pi)\|\varepsilon^j \leq \frac{1}{2}\left(\frac{\varepsilon}{\mu}\right)^{p}.
\intertext{Upper-bounding the polynomial on the LHS using a geometric series and solving for $\varepsilon$ we get}
    \varepsilon =  \max_{j \leq p - 1} \mu \left(\frac{4}{j!}  \norm{\partial^j_x \Delta_t^{\pi_\star}(\xi; \pi)}\right)^{1/p},
\end{align*}
provided that $\varepsilon \leq \alpha$, $2\frac{L_{\partial^p \pi}}{(p+1)!} \varepsilon^{p+1} +  \paren{\frac{\varepsilon}{\mu}}^{p} \leq \eta$, and $\|\partial_x^p \Delta_t^{\pi_\star}(\xi; \pi)\| < \frac{p!}{2\mu^p}$. These conditions correspond to that of the theorem, completing the proof.
\end{proof}

\begin{corollary}
\label{appendix: r equals 1 corollary} Consider a $\delta$-ISS $\fcl{\pi_\star} $system with $\gamma(x) := \gamma x, \gamma > 0$ and $\eta = \infty$. Let Assumption \ref{assumption: k-differentiable Pi} hold with $p=1$ and assume without loss of generality $\gamma L_{\partial \pi} \geq 1$. Provided
\begin{align*}
    \max_{0\leq t \leq T-1} \|\partial_x \Delta_t^{\pi_\star}(\xi; \pi)\| \leq \frac{1}{4\gamma}, \: \max_{0\leq t \leq T-1}\|\Delta_t^{\pi_\star}(\xi; \pi)\| \leq \frac{1}{16\gamma^2 L_{\partial \pi}}
\end{align*}
then for all $0 \leq t \leq T$
\begin{align*}
    \|\state_t^{\pi_\star}(\xi) - \state_t^{\pi}(\xi)\| \leq \max_{0 \leq k \leq t - 1} 8\gamma \|\Delta_k^{\pi_\star}(\xi; \pi)\|.
\end{align*}
\end{corollary}
\begin{proof}
Choose $\alpha := \frac{1}{2\gamma L_{\partial\pi}}$ and $\mu := 2\gamma$. Assume $\gamma L_{\partial\pi} \geq 1$. Since $\gamma^{-1}(x) = \frac{x}{\gamma}$, for $x \leq \alpha$ it holds that
$$L_{\partial^p \pi}x^2 + (x/\mu) \leq \frac{x}{2\gamma} + \frac{x}{2\gamma} \leq \gamma^{-1}(x) := \frac{x}{\gamma}.$$
and we can directly apply the $p = r = 1$ special case of Theorem \ref{thm: slow decay exact r}. Then, if the constraints described by Equations~\eqref{eq: p=r lower derivatives} and \eqref{eq: p=r p derivative} are satisfied:
\begin{align*}
    \max_{0\leq t \leq T-1} \|\partial_x \Delta_t^{\pi_\star}(\xi; \pi)\| \leq \frac{p!}{2\mu^p} = \frac{1}{4\gamma}, \:\; \max_{0\leq t \leq T-1}\|\Delta_t^{\pi_\star}(\xi; \pi)\| \leq \frac{0!}{4} \paren{\frac{\alpha}{\mu}}^{p} = \frac{1}{16\gamma^2 L_{\partial \pi}},
\end{align*}
it holds for all $1\leq t \leq T$
\begin{align*}
    \|\state_t^{\pi_\star}(\xi) - \state_t^{\pi}(\xi)\| \leq \max_{0 \leq k \leq t - 1} 8\gamma \|\Delta_k^{\pi_\star}(\xi; \pi)\|.
\end{align*}
\end{proof}

\section{Proofs for Section~\ref{sec: generalization bounds}}\label{appendix: generalization bounds proofs}

\subsection{Preliminaries}

Let $\calG \subset \R^\calX$ be a set of functions,
and let $x_1, \dots, x_n \in \calX$ be a fixed set of points.
We will endow $\calG$ with the following empirical $L^2$ pseudo-metric space structure:
\begin{align*}
    d(f, g) := \sqrt{\frac{1}{n}\sum_{i=1}^{n} (f(x_i) - g(x_i))^2}, \:\: f,g \in \calG.
\end{align*}
The empirical Rademacher complexity of $\calG$
is defined as:
\begin{align*}
    \calR_n(\calG) := \Ex_{\varepsilon} \brac{\sup_{g \in \calG} \frac{1}{n} \sum_{i=1}^n \varepsilon_i g(x_i)},
\end{align*}
where the $\{\varepsilon_i\}_{i=1}^{n}$ are independent Rademacher random variables.
Dudley's inequality yields a bound on $\calR_n(\calG)$
using the metric space structure of $(\calG, d)$.
\begin{lemma}[{Dudley's inequality~\citep[cf.][Lemma A.3]{srebro2010smoothness}}]
\label{lemma:dudley}
Let $R := \sup_{f \in \calG} d(f,0)$ be the radius of the set
$\calG$. We have that:
\begin{align*}
    \calR_n(\calG) \leq \inf_{\alpha \in [0, R]} \left\{ 4\alpha + \frac{12}{\sqrt{n}} \int_\alpha^R \sqrt{\log N(\calG; d, \varepsilon)} d\varepsilon \right\}.
\end{align*}
Here, $N(\calG;d,\varepsilon)$ denotes the covering number
of $\calG$ in the metric $d$ at resolution $\varepsilon$.
\end{lemma}

\subsection{Generalization bound for the non-realizable setting}

We use standard techniques to derive a generalization bound for the \textit{non-realizable setting}, i.e.,\ where $\pi_\star$ may not necessarily be contained in the hypothesis class $\Pi$. Let $\calG \subset [0, 1]^\calX$ be a given function class. We have the following standard uniform convergence generalization bound~\citep[cf.][Theorem 4.10]{wainwright2019high}: with probability greater than $1 - \delta$ over $x_1,\dots,x_n \overset{\mathrm{i.i.d.}}{\sim} \calD$, we have
\begin{equation}\label{eq: slow rates uniform convergence}
    \sup_{g \in \calG} \abs{\Ex_x\brac{g} - \Ex_n\brac{g} } \leq 2 \Ex_{x_{1:n}}\brac{\calR_n(\calG)} + \sqrt{\frac{\log(2/\delta)}{n}},
\end{equation}
where $\Ex_{x_{1:n}}$ denotes expectation over
the randomness of $x_1,\dots,x_n$.
To establish an upper bound on $\Ex_{x_{1:n}}[\calR_n(\calG)]$, we focus on the Lipschitz parametric case, though we note many analogous bounds can be computed for a plethora of other function classes \citep{wainwright2019high}.
\begin{theorem}\label{thm: slow rates gen bound}
Let $\calG \subset [0, 1]^\calX$ be a $(B_\theta, L_\theta, q)$-Lipschitz parametric function class.
Given $\delta \in (0, 1)$,
with probability at least $1-\delta$ over the i.i.d.\ draws $x_1, \dots, x_n \sim \calD$, the following bound holds:
\begin{align}\label{eq: local rademacher gen bound}
    \sup_{g \in \calG} \abs{\Ex_x\brac{g} - \Ex_n\brac{g} } \leq 48\sqrt{\frac{q \log\paren{3 B_\theta L_\theta}}{n}} + \sqrt{\frac{\log(2/\delta)}{n}}.
\end{align}
\end{theorem}
\begin{proof}
This argument is fairly standard.
Fix a set of points $x_1, \dots, x_n \in \calX$.
Since $\calG$ contains only functions with range
$[0, 1]$, the radius of the set $\calG$
in the empirical $L^2$ metric is:
\begin{align*}
    \sup_{f\in \calG} d(f, 0) &\leq 1.
\end{align*}
Therefore, Dudley's inequality (Lemma~\ref{lemma:dudley}) yields:
\begin{align*}
    \calR_n(\calG) \leq \frac{12}{\sqrt{n}} \int_{0}^{1} \sqrt{\log N(\calG; d, \varepsilon)} d\varepsilon.
\end{align*}
Now using the fact that $\calG$ is a $(B_\theta, L_\theta, q)$-Lipschitz parametric function class, 
it is not hard to see that for any $\varepsilon > 0$,
an $\varepsilon/(B_\theta L_\theta)$-cover of $\mathbb{B}_2^q(1)$
in the Euclidean metric
yields an $\varepsilon$-cover of $\calG$ in the $d$-metric. Hence,
for any $\varepsilon \in (0, 1)$, by a standard volume comparison argument:
\begin{align*}
    \log{N(\calG; d, \varepsilon)} &\leq \log N\paren{\mathbb{B}_2^q(1);\norm{\cdot},\frac{\varepsilon}{B_\theta L_\theta}} \\
    &\leq q \log\paren{1 + \frac{2B_\theta L_\theta}{\varepsilon}} \\
    &\leq q \log\paren{\frac{3B_\theta L_\theta}{\varepsilon}}.
\end{align*}
Therefore, we have:
\begin{align*}
    \int_{0}^{1} \sqrt{\log N(\calG; d, \varepsilon)} d\varepsilon &\leq \sqrt{q} \int_{0}^{1} \sqrt{\log\paren{\frac{3B_\theta L_\theta}{\varepsilon}}} d\varepsilon \\
    &\leq \sqrt{q \log(3B_\theta L_\theta)} + \sqrt{q} \int_0^1 \sqrt{\log(1/\varepsilon)} d \varepsilon && \text{using } \sqrt{a+b} \leq \sqrt{a}+\sqrt{b} \\
    &\leq \sqrt{q \log(3B_\theta L_\theta)} + \sqrt{q} &&\text{using } \int_0^1\sqrt{\log\paren{\frac{1}{\varepsilon}}}d\varepsilon \leq 1\\
    &\leq 2 \sqrt{q\log(3B_\theta L_\theta)}.
\end{align*}
Plugging this back into Dudley's inequality:
\begin{align*}
    \calR_n\paren{\calG} &\leq 24\sqrt{\frac{q}{n}} \sqrt{\log(3B_\theta L_\theta)}.
\end{align*}
The claim now follows from the standard uniform convergence inequality~\eqref{eq: slow rates uniform convergence}.
\end{proof}

Applying this generalization bound to the $(B_\theta, B_{\ell,p}^{-1}L_{\ell,p}, q)$-Lipschitz parametric function class $B_{\ell,p}^{-1}(\ell^{\pi_\star}_{p}\circ\Pi_{\theta, p})$, we get the non-realizable analogue to Corollary~\ref{cor: TaSIL generalization bound}.
\begin{corollary}\label{cor: slow rates TaSIL generalization bound}
    Let the policy class $\Pi_{\theta,p}$ be defined as in \eqref{eq: Pi}.  Let the function class $\ell_{p}^{\pi_\star} \circ \Pi_{\theta, p}$ be defined as in~\eqref{def: p-diff param loss class}, and constants $B_{\ell, p}$, $L_{\ell, p}$ be defined as above.
    Let $\hat{\pi}_{\mathsf{TaSIL},p}$ be any empirical risk minimizer~\eqref{eq: TaSIL ERM}.
    Then with probability at least $1 - \delta$ over the initial conditions $\curly{\xi_i}_{i=1}^n \overset{\mathrm{i.i.d.}}{\sim} \calD^n$,
    \begin{align}\label{eq: slow rates TaSIL generalization bound}
        \Ex_{\xi} \brac{\ell^{\pi_\star}_{p}(\xi; \hat{\pi}_{\mathsf{TaSIL},p})} \leq \Ex_n[\ell^{\pi_\star}_{p}(\cdot\;; \hat{\pi}_{\mathsf{TaSIL},p})] + 48 B_{\ell, p} \sqrt{\frac{q \log\paren{3B_\theta B_{\ell, p}^{-1} L_{\ell, p}}}{n}} + B_{\ell, p} \sqrt{\frac{\log(2/\delta)}{n}}.
    \end{align}
\end{corollary}

Inserting the generalization bound in Corollary~\ref{cor: slow rates TaSIL generalization bound} in lieu of Corollary~\ref{cor: TaSIL generalization bound} for the rest of the bounds seen in Section~\ref{sec: generalization bounds} yields the sample complexity bounds relevant to our problem in the non-realizable setting. However, we note an important subtlety that manifests in the non-realizable regime. We note that in Corollary~\ref{cor: TaSIL generalization bound}, due to realizability, the generalization bound monotonically decreases to $0$ with $n$, whereas in Corollary~\ref{cor: slow rates TaSIL generalization bound}, we have an additive factor of $\Ex_n[\ell^{\pi_\star}_{p}(\cdot\;; \hat{\pi}_{\mathsf{TaSIL},p})]$. It is therefore possible for either small enough $n$ or insufficiently expressive function classes $\Pi_{\theta, p}$ that the non-zero empirical risk automatically violates the imitation error requirements in Theorems~\ref{thm: sub-linear imitation gap} and~\ref{thm: sup-linear imitation gap}. Thus, a necessary assumption must be made in the non-realizable setting for the function class to be expressive enough such that the empirical risk it incurs on sufficiently large datasets satisfies the imitation error requirements with high probability.

\subsection{Proof of Theorem~\ref{thm: local rademacher gen bound}}

Before turning to the proof of Theorem~\ref{thm: local rademacher gen bound}, we introduce some notation and tools from the local Rademacher complexity literature~\citep{bartlett2002rademacher,bousquet2002concentration}.

\begin{definition}[Sub-root function]
\label{def:sub_root}
A function $\phi : [0, \infty) \rightarrow \R$ is said to be a \emph{sub-root function} if:
\begin{enumerate}[label=\alph*)]
    \item $\phi$ is non-negative.
    \item $\phi$ is not the zero function.
    \item $\phi$ is non-decreasing.
    \item $r \mapsto \phi(r)/\sqrt{r}$ is non-increasing.
\end{enumerate}
\end{definition}

For any non-negative function class $\calG$,
scalar $r \geq 0$,
and $n$ points $x_1, \dots, x_n \in \calX$, define:
\begin{align*}
    \calH_n(r; x_{1:n}) := \{ g \in \calG \mid \Ex_n [g] \leq r \}.
\end{align*}

The following is from~\citet{bousquet2002concentration}.
\begin{theorem}[{\citet[Theorem 6.1]{bousquet2002concentration}}]
\label{stmt:bousquet_nonneg}
Let $\calG \subset [0, 1]^\calX$, and fix a $\delta \in (0, 1)$.
With probability at least
$1-\delta$ over the i.i.d.\ draws of $x_1, \dots, x_n$, the following holds.
Let $\phi_n$ be any sub-root function (cf.~Definition~\ref{def:sub_root}) satisfying:
\begin{align*}
    \calR_n(\calH_n(r; x_{1:n})) \leq \phi_n(r), \:\: \forall\, r > 0.
\end{align*}
Let $r_n^*$ denote the largest solution to the equation
$\phi_n(r) = r$.
Then, for all $g\in \calG$:
\begin{align*}
    \Ex_x\brac{g} \leq 2 \Ex_n\brac{g} + 106 r_n^* + \frac{48(\log(1/\delta) + 6 \log\log{n})}{n}.
\end{align*}
\end{theorem}

With these definitions and preliminary results in place,
we turn to the proof of Theorem~\ref{thm: local rademacher gen bound}.
\localrademachergenbound*
\begin{proof}
Fix a set of points $x_1, \dots, x_n \in \calX$.
Define $\calG_n(r; x_{1:n})$ as:
\begin{align*}
    \calG_n(r; x_{1:n}) := \{ g\in \calG \mid \Ex_n[g^2] \leq r \}.
\end{align*}
For what follows, we often suppress the explicit dependence
on $x_{1:n}$ in the notation for $\calH_n$ and $\calG_n$.
Observe that since
$\calG \subset [0, 1]^{\calX}$, we have
$\Ex_n[g^2] \leq \Ex_n\brac{g}$ for every $g \in \calG$, 
and therefore: 
\begin{align*}
    \calH_n(r) \subseteq \calG_n(r), \:\: \forall\, r \geq 0.
\end{align*}
Hence $\calR_n(\calH_n(r)) \leq \calR_n(\calG_n(r))$, and it suffices for us to prove an upper bound on the latter.

\begin{proposition}
\label{stmt:dudley_bound}
Let $\calG \subset [0, 1]^\calX$ be a $(B_\theta, L_\theta, q)$-Lipschitz parametric function class.
Fix a set of points $x_1, \dots, x_n \in \calX$.
We have that:
\begin{align*}
    \calR_n(\calG_n(r; x_{1:n})) \leq 24\sqrt{2} \sqrt{\frac{q}{n}} \min\{\sqrt{r},1\} \sqrt{\log\left(\frac{6 B_\theta L_\theta}{\min\{\sqrt{r},1\}}\right)}.
\end{align*}
\end{proposition}
\begin{proof}[Proof of Proposition~\ref{stmt:dudley_bound}]
The radius of the set $\calG_n(r)$
in the empirical $L^2$ metric $d$ is upper bounded by
$\sqrt{r}$ by definition.
Furthermore, the radius of
$\calG$ in the metric $d$ is upper bounded by one.
Hence, since $\calG_n(r) \subseteq \calG$,
the radius of $\calG_n(r)$ is upper bounded
by $\min\{\sqrt{r},1\}$.

Dudley's inequality~(Lemma~\ref{lemma:dudley}) yields:
\begin{align}
   \calR_n(\calG_n(r)) \leq 
   \inf_{\alpha \in [0, \min\{\sqrt{r},1\}]}
   \left\{
   4\alpha + \frac{12}{\sqrt{n}}\int_\alpha^{\min\{\sqrt{r},1\}} \sqrt{\log{N(\calG; d, \varepsilon/2)}}d\varepsilon
   \right\}.
    \label{eq: F-starhull dudley bound}
\end{align}
Here, we have used the fact that the inclusion
$\calG_n(r) \subseteq \calG$
implies $N(\calG_n(r);d,\varepsilon) \leq N(\calG;d,\varepsilon/2)$ by
\citet[Exercise 4.2.10]{vershyninHDP}.

% {\color{red}TODO: this line is not correct, need to 
% change $\varepsilon$ to $\varepsilon/2$
% via Exercise 4.2.10 of Vershynin's HDP book.
% This will not affect the final result,
% but the intermediate constants in the proof
% will need to change.}

Since $\calG$ is $(B_\theta,L_\theta,q)$-Lipschitz, for any $\varepsilon > 0$,
an $\varepsilon$-covering of $\calG$ in the $d$-metric
can be constructed from an
$\varepsilon/(B_\theta L_\theta)$-covering of
$\mathbb{B}_2^q(1)$ in the Euclidean metric. 
Therefore, for any $\varepsilon \in (0, 1)$,
by the standard volume comparison bound:
\begin{align*}
    \log{N(\calG;d,\varepsilon)} &\leq \log{N\paren{\mathbb{B}_2^q(1);\norm{\cdot},\frac{\varepsilon}{B_\theta L_\theta}}} \\
    &\leq q \log\paren{1 + \frac{2B_\theta L_\theta}{\varepsilon}} \\
    &\leq q \log\paren{\frac{3B_\theta L_\theta}{\varepsilon}}.
\end{align*}
Putting $R:=\min\{\sqrt{r},1\}$,
\begin{align*}
    &\int_0^{R} \sqrt{\log{N(\calG;d,\varepsilon/2)}} d\varepsilon \\
    &\leq \sqrt{q} \left[ R \sqrt{\log(6B_\theta L_\theta)} + \int_0^R \sqrt{\log(1/\varepsilon)} d\varepsilon \right] && \text{using } \sqrt{a+b}\leq\sqrt{a}+\sqrt{b}\\
    &= \sqrt{q}\left[ R \sqrt{\log(6B_\theta L_\theta)} + R \int_0^1 \sqrt{\log\paren{\frac{1}{R\varepsilon}}} d\varepsilon \right] &&\text{change of variables } \varepsilon \gets \varepsilon/R \\
    &\leq \sqrt{q}\left[ R \sqrt{\log(6B_\theta L_\theta)} + R \sqrt{\log\paren{\frac{1}{R}}} + R \right] &&\text{using } \int_0^1 \sqrt{\log\paren{\frac{1}{\varepsilon}}} d\varepsilon \leq 1 \\
    &\leq R\sqrt{q} \left[ \sqrt{\log(6B_\theta L_\theta)} + 2\sqrt{\log\paren{\frac{1}{R}}} \right] \\
    &\leq 2\sqrt{2} R \sqrt{q} \sqrt{\log\paren{\frac{6 B_\theta L_\theta}{R}}} &&\text{using } \sqrt{a}+\sqrt{b}\leq \sqrt{2} \sqrt{a+b}.
\end{align*}

The claim now follows.
\end{proof}

We complete the proof by upper bounding $r_n^*$
and invoking Theorem~\ref{stmt:bousquet_nonneg}.
First, observe that by Cauchy-Schwarz, the inequality
$\Ex_n [g^2] \leq \Ex_n \brac{g}$ for $g \in \calG$, and
Jensen's inequality:
\begin{align*}
  \calR_n(\calH_n(r)) \leq \sup_{g\in \calH_n(r)} \sqrt{ \Ex_n[g^2] } \Ex_{\varepsilon} \sqrt{\frac{1}{n} \sum_{i=1}^{n} \varepsilon_i^2} \leq \sqrt{r}.
\end{align*}
This bound holds for any $r \geq 0$.
Hence, when $r \leq 1/n^2$:
\begin{align*}
    \calR_n(\calH_n(r)) \leq 1/n.
\end{align*}
On the other hand, when $r > 1/n^2$, by $\calR_n(\calH_n(r)) \leq \calR_n(\calG_n(r))$, Proposition~\ref{stmt:dudley_bound},
and the inequalities $1/n < \min\{\sqrt{r},1\}\leq\sqrt{r}$:
\begin{align*}
    \calR_n(\calH_n(r)) &\leq 24\sqrt{2}\sqrt{\frac{q}{n}} \sqrt{r} \sqrt{\log\paren{ 6B_\theta L_\theta n }}.
\end{align*}
Hence, the function $\phi_n$ defined as:
\begin{align*}
    \phi_n(r) := \max\left\{ 24\sqrt{2} \sqrt{\frac{q \log(6B_\theta L_\theta n)}{n}} \sqrt{r}, \frac{1}{n}\right\},
\end{align*}
satisfies $\calR_n(\calH_n(r)) \leq \phi_n(r)$ for all $r \geq 0$.
It is also not hard to see that $\phi_n$ is
a sub-root function (cf.~Definition~\ref{def:sub_root}).
Therefore, there is a unique solution $r_n^*$ satisfying
$\phi_n(r_n^*) = r_n^*$.
Now, for any positive constants $A, B$,
the root of $r = \max\{A \sqrt{r}, B\}$ is upper bounded by
$\max\{A^2,B\}$.
Hence,
\begin{align*}
    r_n^* \leq 1152\frac{q\log(6B_\theta L_\theta n)}{n}.
\end{align*}

Theorem~\ref{thm: local rademacher gen bound} now follows by Theorem~\ref{stmt:bousquet_nonneg}.
\end{proof}

\subsection{Proof of Corollary~\ref{cor: TaSIL generalization bound}}

\begin{lemma}\label{lem: loss params}
Let $B_{\ell,p}:= \frac{2}{p+1}\textstyle\sum_{j=0}^{p}B_{j}$ and $L_{\ell,p}:=\frac{B_X}{p+1}\textstyle\sum_{j=0}^{p}L_{j}$.  Then $B_{\ell,p}^{-1}(\ell^{\pi_\star}_{p}\circ\Pi_{\theta, p})$ is a $(B_\theta, B_{\ell,p}^{-1}L_{\ell,p}, q)$-Lipschitz parametric function class
\end{lemma}
\begin{proof}
It suffices to show that
\begin{align*}
    \max_{0\leq t \leq T-1}  \left\|\partial^j_x \Delta^{\pi_\star}_{t}(\xi;\pi) \right\|
\end{align*}
is $2B_j$-bounded and $B_X L_j$ Lipschitz with respect to $\Theta$. By definition, we immediately get
\begin{align*}
    \left\|\partial^j_x \Delta^{\pi_\star}_{t}(\xi;\pi) \right\| &= \norm{\partial^j_x \pi_\star(\state_t^{\pi_\star}(\xi)) - \partial^j_x\pi(\state_t^{\pi_\star}(\xi)) } \\
    &\leq 2 \sup_{\norm{x}\leq B_X, \norm{\theta}\leq B_\theta}\norm{\partial^j_x\pi(x, \theta)} \\
    &= 2B_j.
\end{align*}
To bound the Lipschitz constant, we iteratively apply the Fundamental Theorem of Line Integrals:
\begin{align*}
    \partial^j_x \pi(x; \theta_1) - \partial^j_x \pi(x; \theta_2) &= \int_{\theta_2}^{\theta_1} \int_0^x \frac{\partial^{j+2} \pi}{\partial x^{j+1} \partial \theta} (z \otimes \omega) \;dz d\omega \\
    &= \int_{\theta_2}^{\theta_1} \paren{ \int_0^1 \frac{\partial^{j+2} \pi}{\partial x^{j+1} \partial \theta}(\alpha x \otimes \omega) \;d\alpha} x \;d\omega \\
    &= \paren{\int_{0}^{1} \int_0^1 \frac{\partial^{j+2} \pi}{\partial x^{j+1} \partial \theta}(\alpha x \otimes (\theta_2 + \beta(\theta_1 - \theta_2))) \;d\alpha d\beta} x \otimes (\theta_1 - \theta_2).
\end{align*}
Taking norms on both sides, we get
\begin{align*}
    \norm{\partial^j_x \pi(x; \theta_1) - \partial^j_x \pi(x; \theta_2)} &\leq \sup_{\norm{x}\leq B_X, \norm{\theta}\leq B_\theta}\norm{\frac{\partial^{j+2}\pi}{\partial x^{j+1} \partial \theta}} \norm{x} \norm{\theta_1 - \theta_2} \\
    &\leq B_X L_j \norm{\theta_1 - \theta_2},
\end{align*}
which establishes that $\left\|\partial^j_x \Delta^{\pi_\star}_{t}(\xi;\pi) \right\|$ is $B_X L_j$-Lipschitz. Recalling that
\begin{equation*}
    \ell^{\pi_\star}_{p}(\xi; \pi):= \tfrac{1}{p+1}\textstyle\sum_{j=0}^{p} \max_{0\leq t \leq T-1}  \left\|\partial^j_x \Delta^{\pi_\star}_{t}(\xi;\pi) \right\|,
\end{equation*}
it follows that $\ell^{\pi_\star}_{p}(\xi; \pi)$ is $\frac{2}{p+1}\textstyle\sum_{j=0}^{p}B_{j}$-bounded and $\frac{B_X}{p+1}\textstyle\sum_{j=0}^{p}L_{j}$-Lipschitz. 
\end{proof}

\tasilgeneralizationbound*
\begin{proof} 
This follows by directly applying the constants derived in Lemma~\ref{lem: loss params} to Theorem~\ref{thm: local rademacher gen bound}, and using the assumption that $\pi_\star \in \Pi_{\theta, p}$ such that $\Ex_n\brac{\ell_p^{\pi_\star}(\cdot \;; \hat{\pi}_{\mathsf{TaSIL}, p})} = 0$.
\end{proof}

\subsection{Proofs of Theorem~\ref{thm: final fast} and Theorem~\ref{thm: final slow}}\label{appendix: TaSIL sample complexity proofs}

Before proceeding to the proofs of the main sample complexity bounds, we introduce the following lemma for inverting functions of the form $\log{n}/n$, adapted from~\citet[Lemma A.4]{simchowitz2018learning}.
% \Stephen{TODO: cite simchowitz}
\begin{lemma}\label{lem: inverting gen bounds}
    Given $n \in \N$, $n \geq b \log(c n)$ as long as $n \geq 2 b \log(2 b c)$, where we assume $b,c \geq 1$.
\end{lemma}
\begin{proof}
We observe by derivatives that $n - b\log(cn)$ is strictly increasing for $n \geq b$. Therefore, it suffices to show $b \log(cn) \leq n$ when $n = 2b \log(2bc)$.
\begin{align*}
    b\log(2bc\log(2bc)) &= b\log\paren{2\log(2)bc + 2bc\log(bc)} \\
    &\leq b\log\paren{\paren{2\log(2) + 2}(bc)^2} &&bc \geq 1\\
    &= 2b\log\paren{\sqrt{2\log(2) + 2}bc} \\
    &< 2b\log\paren{2bc}.
\end{align*}
\end{proof}

\begin{theorem}[Full version of Theorem~\ref{thm: final fast}]
Assume that $\pi_\star \in \Pi_{\theta,0}$ and let the assumptions of Theorem \ref{thm: sub-linear imitation gap} hold for all $\pi\in\Pi_{\theta,0}$. Let Equation~\eqref{eq: sub-linear neighborhood relation} hold with constants $\mu, \alpha > 0$, and assume without loss of generality that $\alpha/\mu \leq 1$, $L_\pi \mu \geq 1/2$.
    Let $\hat{\pi}_{\mathsf{TaSIL},0}$ be an empirical risk minimizer of $\ell^{\pi_\star}_{0}$ over the policy class $\Pi_{\theta,0}$ for initial conditions $\curly{\xi_i} \overset{\mathrm{i.i.d.}}{\sim} \calD^n$. % and assume that we are in the realizable regime, such that $\Ex_n[g(\cdot, \hat{\pi}_{\mathsf{TaSIL},0})] = 0$.
    Fix a failure probability $\delta \in (0,1)$, and assume that 
    \begin{align*}
    n \geq \calO(1) \max\curly{B_{\ell, 0}  \frac{\kappa_\alpha}{\delta}\log\paren{\frac{\kappa_{\alpha}B_\theta B_{\ell, 0}^{-1} L_{\ell, 0}}{\delta} }, \; B_{\ell, 0}  \frac{\kappa_\eta}{\delta} \log\paren{\frac{\kappa_\eta B_\theta B_{\ell, 0}^{-1} L_{\ell, 0}}{\delta} }},
    \end{align*}
    where $\kappa_\alpha :=  q(\mu/\alpha)^{\tfrac{1}{1 + r}}$, $\kappa_{\eta} := q L_\pi \mu/\eta$. Then with probability at least $1-\delta$, the imitation gap evaluated on $\xi \sim \calD$ (drawn independently from $\curly{\xi_i}_{i=1}^n$) satisfies
    \begin{align*}
       \Gamma_T(\xi; \hat{\pi}_{\mathsf{TaSIL},0}) \leq \calO(1) \mu \left(\frac{1}{\delta} \frac{B_{\ell, 0} q \log\paren{B_\theta B_{\ell, 0}^{-1} L_{\ell, 0} n }}{n}\right)^{1+r}.
    \end{align*}
\end{theorem}
\begin{proof} 
Applying Corollary~\ref{cor: TaSIL generalization bound} to the $(B_{\theta}, B_{\ell, 0}, q)$-Lipschitz parametric function class $B_{\ell, 0}^{-1} \paren{\ell_0^{\pi_\star} \circ \Pi_{\theta, 0}}$, we get that with probability at least $1 - \delta/2$ over i.i.d.\ initial conditions ${\xi_i} \sim \calD^n$,
\[
\Ex_{\xi}\brac{\max_{0\leq t \leq T-1}  \left\|\Delta^{\pi_\star}_{t}(\xi;\hat{\pi}_{\mathsf{TaSIL},0}) \right\|} \leq \calO(1) B_{\ell, 0}\frac{ q \log\paren{B_\theta B_{\ell,0}^{-1} L_{\ell, 0} n} + \log(1/\delta)}{n}.
\]
Applying Markov's inequality to $\max_t \norm{\Delta^{\pi_\star}_{t}(\xi;\hat{\pi}_{\mathsf{TaSIL},0})}$, for a new draw $\xi \sim \calD$, with probability greater than $1 - \delta / 2$,
\[
\max_{0\leq t \leq T-1}  \norm{\Delta^{\pi_\star}_{t}(\xi;\hat{\pi}_{\mathsf{TaSIL},0})} \leq \frac{2}{\delta}\Ex_{\xi}\brac{\max_{0\leq t \leq T-1}  \left\|\Delta^{\pi_\star}_{t}(\xi;\hat{\pi}_{\mathsf{TaSIL},0}) \right\|}.
\]
Thus applying a union bound over the two events, we have with probability greater than $1 - \delta$ that
\begin{equation}\label{eq: ell_0 gen bound}
\max_{0\leq t \leq T-1}  \norm{\Delta^{\pi_\star}_{t}(\xi;\hat{\pi}_{\mathsf{TaSIL},0})} \leq \calO(1) B_{\ell, 0}q \frac{1}{\delta}\frac{\log\paren{B_\theta B_{\ell,0}^{-1} L_{\ell, 0} n} + \log(1/\delta)}{n},
\end{equation}
where we absorb numerical constants into $\calO(1)$. We want $\max_t \norm{\Delta^{\pi_\star}_{t}(\xi;\hat{\pi}_{\mathsf{TaSIL},0})}$ to satisfy the conditions in~\eqref{eq: sub-linear closeness alpha}; that is,
\begin{align*}
    &\max_{0 \leq t \leq T-1} \mu\norm{\Delta_t^{\pi_\star}(\xi; \hat{\pi}_{\mathsf{TaSIL},0})}^{1+r} \leq \alpha, \\
    &\max_{0 \leq t \leq T-1} 2L_\pi \mu \norm{\Delta_t^{\pi_\star}(\xi; \hat{\pi}_{\mathsf{TaSIL},0})}^{1+r} + \norm{\Delta_t^{\pi_\star}(\xi; \hat{\pi}_{\mathsf{TaSIL},0})} \leq \eta.
\end{align*}
For notational convenience, we further require $\max_t \norm{\Delta_t^{\pi_\star}(\xi; \hat{\pi}_{\mathsf{TaSIL},0})} \leq 1$, so that $$\max_t \norm{\Delta_t^{\pi_\star}(\xi; \hat{\pi}_{\mathsf{TaSIL},0})}^{1+r} \leq \max_t \norm{\Delta_t^{\pi_\star}(\xi; \hat{\pi}_{\mathsf{TaSIL},0})}.$$ 
By assumption, since $\alpha/\mu \leq 1$, satisfying the first condition above implies $\max_t \norm{\Delta_t^{\pi_\star}(\xi; \hat{\pi}_{\mathsf{TaSIL},0})} \leq 1$. We observe that for $n \geq \delta^{-1}\log(1/\delta)$ we have $\log{n} \geq 2\log\paren{1/\delta}$, thus it suffices to absorb the $\log(1/\delta)$ term into $\log{n}$. Inserting the generalization bound~\eqref{eq: ell_0 gen bound} and shifting $n$ to the right-hand side of the above conditions, we have the following requirements on $n$:
\begin{align*}
    n \geq \calO(1) \max\bigg\{
    &\paren{\frac{\mu}{\alpha}}^{1/1+r} B_{\ell, 0}q \frac{1}{\delta} \log\paren{B_\theta B_{\ell,0}^{-1} L_{\ell, 0} n}, \\
    & \paren{\frac{L_\pi \mu}{\eta}} B_{\ell, 0}q \frac{1}{\delta} \log\paren{B_\theta B_{\ell,0}^{-1} L_{\ell, 0} n} \bigg\} \\
    =: \calO(1) \max\bigg\{
    &B_{\ell, 0}\kappa_{\alpha} \frac{1}{\delta} \log\paren{B_\theta B_{\ell,0}^{-1} L_{\ell, 0} n},\\
    &B_{\ell, 0}\kappa_{\eta} \frac{1}{\delta} \log\paren{B_\theta B_{\ell,0}^{-1} L_{\ell, 0} n} \bigg\},
\end{align*}
where we define $\kappa_\alpha = q(\mu/\alpha)^{\tfrac{1}{1 + r}}$ and $\kappa_\eta = qL_\pi\mu/\eta$.
Therefore, applying Lemma~\ref{lem: inverting gen bounds} on each of the arguments of the maximum, setting $b = B_{\ell, 0}\kappa_{\alpha}q/\delta$ (respectively $b = B_{\ell, 0}\kappa_{\eta}q/\delta$) and $c = B_{\theta}B_{\ell, 0}^{-1} L_{\ell, 0}$, we get the following sample complexity bounds. For $n$ satisfying
\begin{align*}
    n \geq \calO(1) \max\curly{
    B_{\ell, 0}  \frac{\kappa_\alpha}{\delta}\log\paren{\frac{\kappa_{\alpha}B_\theta  L_{\ell, 0}}{\delta} }, \; B_{\ell, 0}  \frac{\kappa_\eta}{\delta} \log\paren{\frac{\kappa_\eta B_\theta L_{\ell, 0}}{\delta} }},
\end{align*}
we have with probability greater than $1 - \delta$
\begin{equation*}
    \Gamma_T(\xi; \hat{\pi}_{\mathsf{TaSIL},0}) \leq \calO(1)\; \mu \left(B_{\ell, 0}q\frac{1}{\delta} \frac{ \log\paren{B_\theta B_{\ell, 0}^{-1} L_{\ell, 0} n }}{n}\right)^{1+r}.
\end{equation*}
This completes the proof.
\end{proof}

\begin{theorem}[Full version of Theorem~\ref{thm: final slow}]
    Assume that $\pi_\star \in \Pi_{\theta,p}$, and let the assumptions of Theorem~\ref{thm: sup-linear imitation gap} hold for all $\pi \in \Pi_{\theta,p}$. Let Equation~\eqref{eq: sup-linear neighborhood relation} hold with constants $\mu, \alpha > 0$, and without loss of generality let $\big(\frac{\alpha}{\mu})^r p! \leq 2$. Let $\hat{\pi}_{\mathsf{TaSIL},p}$ be an empirical risk minimizer of $\ell^{\pi_\star}_{p}$ over the policy class $\Pi_{\theta,p}$ for initial conditions $\curly{\xi_i} \overset{\mathrm{i.i.d.}}{\sim} \calD^n$.
    Fix a failure probability $\delta \in (0,1)$, and assume
    \begin{align*}
        n \geq \calO(1) \max_{j \leq p} \max\curly{ B_j\frac{\kappa_{\alpha, j}}{\delta} \log\paren{\frac{\kappa_{\alpha, j} B_\theta B_j^{-1} B_X L_j }{\delta} }, \;
         B_j\frac{\kappa_{\eta, j}}{\delta } \log\paren{\frac{\kappa_{\eta, j} B_\theta B_j^{-1} B_X L_j}{\delta} }},
    \end{align*}
    where $\kappa_{\alpha, j} :=  \big(\frac{\mu}{\alpha}\big)^r\frac{p q }{j!}$ and $\kappa_{\eta, j} := \big(\frac{L_{\partial^p \pi}}{(p+1)!}\frac{\mu^{p + 1}}{\paren{j!}^{\frac{p+1}{r}}} + \frac{1}{j!}\big) \frac{pq}{\eta \delta}$. Then with probability at least $1 - \delta$, the imitation gap evaluated on $\xi \sim \calD$ (drawn independently from $\curly{\xi_i}_{i=1}^n$) satisfies
    \begin{align*}
        \Gamma_T(\xi; \hat{\pi}_{\mathsf{TaSIL},p}) \leq \calO(1)\mu \max_{j \leq p} \left(\frac{p}{j! \delta} \frac{B_{j} q \log\paren{B_\theta B_j^{-1} B_X L_j n }}{n}\right)^{1/r}.
    \end{align*}
\end{theorem}
\begin{proof} 
Let us first define the following losses on a specific partial:
\begin{align*}
    h_j^{\pi_\star}(\xi; \pi) := \max_{0\leq t \leq T-1}  \left\|\partial^j_x \Delta^{\pi_\star}_{t}(\xi;\pi) \right\|.
\end{align*}
We observe that by definition, $h_j^{\pi_\star} \circ \Pi_{\theta, p}$ is $2B_j$ bounded, and $h_j^{\pi_\star}$ is $B_X L_j$-Lipschitz with respect to $\Theta$ for $j \leq p$, such that $0.5 B_j^{-1}\paren{h_j^{\pi_\star} \circ \Pi_{\theta, p}}$ is a $(B_\theta, 0.5B_j^{-1}B_X L_j, q)$-Lipschitz loss class. We note that since $\pi_\star \in \Pi_{\theta, p}$, we have for any dataset $\curly{\xi_i} \subset \calX$
\begin{align*}
    \Ex_{n} \brac{\ell^{\pi_\star}_{p}(\cdot \;; \hat{\pi}_{\mathsf{TaSIL},p})} =: \frac{1}{p+1}\sum_{j=0}^{p} \max_{0\leq t \leq T-1}  \left\|\partial^j_x \Delta^{\pi_\star}_{t}(\xi;\pi) \right\| = 0,
\end{align*}
which therefore implies $\Ex_n \brac{h_j^{\pi_\star}(\cdot \;; \hat{\pi}_{\mathsf{TaSIL},p}) } = 0$ for $j \leq p$. We now apply the same proof structure in Theorem~\ref{thm: final fast} to each $0.5 B_j^{-1}\paren{h_j^{\pi_\star} \circ \Pi_{\theta, p}}$, where we have with probability greater than $1 - \tfrac{\delta}{2(p+1)}$ that
\begin{align*}
    \Ex_{\xi}\brac{\max_{0\leq t \leq T-1}  \left\|\partial^j_x \Delta^{\pi_\star}_{t}(\xi;\hat{\pi}_{\mathsf{TaSIL},p}) \right\|} \leq \calO(1) B_{j}\frac{ q \log\paren{B_\theta B_{j}^{-1} B_X L_{j} n} + \log\paren{\frac{2(p+1)}{\delta}}}{n}.
\end{align*}
Applying Markov's inequality at level $\tfrac{\delta}{2(p+1)}$, we get with total probability greater than $1 - \tfrac{\delta}{p+1}$ over a new initial condition $\xi \sim \calD$ that $\hat{\pi}_{\mathsf{TaSIL},p}$ satisfies the generalization bound
\begin{align}\label{eq: partial j gen bound}
    \max_{0\leq t \leq T-1}  \left\|\partial^j_x \Delta^{\pi_\star}_{t}(\xi;\hat{\pi}_{\mathsf{TaSIL},p}) \right\| \leq \calO(1) B_{j}\frac{p+1}{\delta}\frac{ q \log\paren{B_\theta B_{j}^{-1} B_X L_{j} n} + \log\paren{\frac{p+1}{\delta}}}{n}.
\end{align}
For each partial, we want to satisfy the constraints outlined in~\eqref{eq: sup-linear closeness alpha}:
\begin{align}
    \label{eq: partial j constraints}
    \begin{split}
        &\max_{0 \leq t \leq T-1} \max_{0 \leq j \leq p} \mu\paren{\frac{2}{j!}  \norm{\partial^j_x \Delta_t^{\pi_\star}(\xi; \pi)}}^{1/r} \leq \alpha, \\
        &\max_{0 \leq t \leq T-1} \max_{0 \leq j \leq p} \frac{2 L_{\partial^p \pi}\mu^{p + 1}}{(p+1)!}\paren{\frac{2}{j!} \norm{\partial^j_x\Delta_t^{\pi_\star}(\xi; \pi)}}^{\frac{p+1}{r}} + \frac{2}{j!}  \norm{\partial^j_x \Delta_t^{\pi_\star}(\xi; \pi)} \leq \eta,
    \end{split}
\end{align}
By assumption, we have $\big(\frac{\alpha}{\mu}\big)^r p! \leq 2$, and thus the first condition implies $\max_t \norm{\partial^j_x \Delta_t^{\pi_\star}(\xi; \pi)} \leq 1$ for all $j \leq p$; in particular, this conveniently ensures $\norm{\partial^j_x\Delta_t^{\pi_\star}(\xi; \pi)}^{\frac{p+1}{r}} \leq \norm{\partial^j_x\Delta_t^{\pi_\star}(\xi; \pi)}$. Plugging the earlier generalization bound~\eqref{eq: partial j gen bound} into the above constraints and shifting $n$ to the RHS, and observing like earlier we may absorb the $\log(1/\delta)$ term into the $\log{n}$ term, we get:
\begin{align*}
    n \geq \calO(1)\max\bigg\{& \paren{\frac{\mu}{\alpha}}^r \frac{1}{j!} B_{j}\frac{p}{\delta}q \log\paren{B_\theta B_{j}^{-1} L_{j} n}, \\
    &\paren{\frac{L_{\partial^p \pi}}{(p+1)!}\frac{\mu^{p + 1}}{\paren{j!}^{\frac{p+1}{r}}} + \frac{1}{j!}}B_{j}\frac{p}{\delta}q \log\paren{B_\theta B_{j}^{-1} L_{j} n} \bigg\} \\
    =: \calO(1)\max\bigg\{& B_{j}\frac{\kappa_{\alpha, j}}{\delta} \log\paren{B_\theta B_{j}^{-1} L_{j} n}, B_{j}\frac{\kappa_{\eta, j}}{\delta}q \log\paren{B_\theta B_{j}^{-1} L_{j} n} \bigg\},
\end{align*}
where we define $\kappa_{\alpha, j} = \paren{\frac{\mu}{\alpha}}^r \frac{pq}{j!}$, $\kappa_{\eta, j} = \paren{\frac{L_{\partial^p \pi}}{(p+1)!}\frac{\mu^{p + 1}}{\paren{j!}^{\frac{p+1}{r}}} + \frac{1}{j!}} \frac{pq}{\eta \delta}$. Therefore applying Lemma~\ref{lem: inverting gen bounds}, setting $b = B_j \frac{\kappa_{\alpha, j}}{\delta}$ (respectively $b = B_j \frac{\kappa_{\eta, j}}{\delta}$) and $c = B_\theta B_{j}^{-1} L_{j}$, for $n$ satisfying:
\begin{align*}
    n \geq \calO(1)\max\bigg\{& B_{j}\frac{\kappa_{\alpha, j}}{\delta} \log\paren{\frac{\kappa_{\alpha, j} B_\theta L_{j}}{\delta} }, B_{j}\frac{\kappa_{\eta, j}}{\delta} \log\paren{\frac{\kappa_{\eta, j} B_\theta L_{j}}{\delta}} \bigg\},
\end{align*}
we have with probability greater than $1 - \frac{\delta}{p+1}$ that the conditions~\eqref{eq: partial j constraints} are satisfied. To finish the proof, since we have with probability $1 - \frac{\delta}{p+1}$ that each $j$th partial difference satisfies the necessary conditions, we union bound over $0 \leq j \leq p$, such that we take a maximum over $j$ for the sample complexity and the resulting imitation gap. This gets us with probability greater than $1 - \delta$, for $n$ satisfying
\begin{align*}
    n \geq \calO(1) \max_{j \leq p} \max\bigg\{& B_{j}\frac{\kappa_{\alpha, j}}{\delta} \log\paren{\frac{\kappa_{\alpha, j} B_\theta L_{j}}{\delta} }, B_{j}\frac{\kappa_{\eta, j}}{\delta} \log\paren{\frac{\kappa_{\eta, j} B_\theta L_{j}}{\delta}} \bigg\},
\end{align*}
that the following bound on the imitation gap holds
\begin{align*}
    \Gamma_T(\xi; \hat{\pi}_{\mathsf{TaSIL},p}) \leq \calO(1)\max_{j \leq p} \mu \left(\frac{p}{j! \delta} \frac{B_{j} q \log\paren{B_\theta B_j^{-1} B_X L_j n }}{n}\right)^{1/r}.
\end{align*}
This completes the proof.
\end{proof}

\section{Using finite-differencing to approximate derivatives}
\label{appendix: finite differencing}

\subsection{Satisfying Conditions~\eqref{eq: sup-linear closeness alpha} and~\eqref{eq: sup-linear closeness eta} with approximate derivatives}

We recall the closeness conditions on the partials along expert trajectories that guarantee bounds on the imitation gap:
\begin{align*}
    &\max_{0 \leq t \leq T-1} \max_{0 \leq j \leq p} \mu \paren{\frac{2}{j!}  \norm{\partial^j_x \Delta_t^{\pi_\star}(\xi; \pi)}}^{1/r} \leq \alpha  \tag{\ref{eq: sup-linear closeness alpha}}, \\
    &\max_{0 \leq t \leq T-1} \max_{0 \leq j \leq p} \frac{2 L_{\partial^p \pi}\mu^{p + 1}}{(p+1)!}\paren{\frac{2}{j!} \norm{\partial^j_x\Delta_t^{\pi_\star}(\xi; \pi)}}^{\frac{p+1}{r}} + \frac{2}{j!}  \norm{\partial^j_x \Delta_t^{\pi_\star}(\xi; \pi)} \leq \eta. \tag{\ref{eq: sup-linear closeness eta}}
\end{align*}
If we have access to approximate derivatives of the expert $\widehat{\partial^j_x \pi_\star}(x)$ such that
\[
\norm{\widehat{\partial^j_x \pi_\star}(x) - \partial^j_x \pi_\star(x)} \leq b < 1
\]
for all $x \in \R^d$, then it suffices to tighten the constraints by some function of $b$ such that minimizing with respect to the approximate partial derivatives will still result in the deviation from the true derivatives satisfying the requisite bounds. Let us define
\begin{align*}
    \widehat{\partial^j_x \Delta_t^{\pi_\star}}(\xi; \pi) :=  \partial^j_x \pi(\state_t^{\pi_\star}(\xi)) - \widehat{\partial^j_x \pi_\star}(\state_t^{\pi_\star}(\xi)),
\end{align*}
such that
\begin{align*}
    \norm{\partial^j_x \Delta_t^{\pi_\star}(\xi; \pi)} &\leq \norm{ \widehat{\partial^j_x \Delta_t^{\pi_\star}}(\xi; \pi)} + \norm{\widehat{\partial^j_x \pi_\star}(\state_t^{\pi_\star}(\xi)) - \partial^j_x \pi_\star(\state_t^{\pi_\star}(\xi))} \\
    &\leq \norm{ \widehat{\partial^j_x \Delta_t^{\pi_\star}}(\xi; \pi)} + b.
\end{align*}
Therefore, it suffices to match the approximate partial derivatives such that
\begin{align*}
    &\max_{0 \leq t \leq T-1} \max_{0 \leq j \leq p} \mu\paren{\frac{2}{j!}  \norm{ \widehat{\partial^j_x \Delta_t^{\pi_\star}}(\xi; \pi)}}^{1/r} \leq \hat{\alpha}, \\
    &\max_{0 \leq t \leq T-1} \max_{0 \leq j \leq p} \frac{2 L_{\partial^p \pi}\mu^{p + 1}}{(p+1)!}\paren{\frac{2}{j!} \norm{ \widehat{\partial^j_x \Delta_t^{\pi_\star}}(\xi; \pi)}}^{\frac{p+1}{r}} + \frac{2}{j!}  \norm{ \widehat{\partial^j_x \Delta_t^{\pi_\star}}(\xi; \pi)} \leq \hat{\eta},
\end{align*}
where, provided $\norm{ \widehat{\partial^j_x \Delta_t^{\pi_\star}}(\xi; \pi)} < 1$:
\begin{align*}
    \hat{\alpha} := \paren{\alpha^r - \frac{2\mu^r}{j!}b}^{1/r}, \:\:
    \hat{\eta} := \eta - \paren{\frac{2 L_{\partial^p \pi}\mu^{p + 1}}{(p+1)!}\paren{\frac{2}{j!}}^{\frac{p + 1}{r}} + \frac{2}{j!} } b.
\end{align*}

A similar bound holds if we also do not have access to the exact derivatives of the learned policy. In practice, these bounds tell us qualitatively that if a sufficiently precise estimate of the derivatives is used, such as through finite differencing, then the imitation gap bounds in Theorem~\ref{thm: sup-linear imitation gap} still hold.

\subsection{Practical approaches for approximating derivatives}

Minimizing $\sum_{j=1}^k \|\partial_x^j \Delta_t^{\pi_\star}(\xi; \pi)\|$ can be approximated provided $\pi_\star$ can be evaluated at points $\{\state_t(\xi) + \delta_i\}_{i=1}^N$ by minimizing the finite difference loss:
$$\ell_{p, \mathsf{FD}}(\xi; \pi, \{\delta_i\}_{i=1}^{N}) := \max_{1 \leq i \leq N} \|\pi_\star(\state_t(\xi) + \delta_i) - \pi_\star(\state_t(\xi)) - \paren{\pi(\state_t(\xi) + \delta_i) - \pi(\state_t(\xi)}\|,$$
where the $\{\delta_i\}$ are chosen such that the Taylor expansion 
\begin{align*}
     \sum_{j=1}^p \frac{1}{p!}\partial^j_x \Delta_t^{\pi_\star}(\xi; \pi) \cdot \delta_i^{\otimes j} &= \pi_\star(\state_t(\xi) + \delta_i) - \pi_\star(\state_t(\xi)) \\
    &\qquad- \left(\pi(\state_t(\xi) + \delta_i) - \pi(\state_t(\xi))\right) - \frac{R_{p + 1}(\delta_i)}{(p + 1)!}, \: \forall \: 1 \leq i \leq N,
\end{align*}
forms a linearly independent system of equations in the derivative parameters.
Here, $R_{p+1}(\delta_i)$ denotes the Taylor remainder,
which satisfies the inequality
$\|R_{p + 1}(\delta_i)\| \leq 2L_{\partial^p\pi}\|\delta_i\|^{p + 1}$ by Assumption~\ref{assumption: k-differentiable Pi}.

For the case $p = 1$, we can stack the $\{\delta_i\}$ into a matrix $S$, the finite differences into a matrix $M$ and the remainders into a matrix $R$ to write
\begin{align*}
    \partial^j_x \Delta_t^{\pi_\star}(\xi; \pi) S = M - R.
\end{align*}
Provided the $\{\delta_i\}$ are chosen such that $S$ is invertible, the operator norm of $\partial^j_x \Delta_t^{\pi_\star}(\xi; \pi)$ can be upper bounded
\begin{align*}
    \|\partial^j_x \Delta_t^{\pi_\star}(\xi; \pi)\| &= \|M S^{-1} - RS^{-1}\| \\
    &\leq \|S^{-1}\|(\|M\| + \|R\|) \\
    &\leq \|S^{-1}\|(\|M\| + L_{\partial^p\pi}\|S\|^{2}).
\end{align*}
For instance, using a standard basis $S = \varepsilon I$ as the finite difference perturbations yields the following bound on the operator norm:
\begin{align*}
    \|\partial^j_x \Delta_t^{\pi_\star}(\xi; \pi)\| \leq \frac{1}{\varepsilon}\|M\| + \varepsilon L_{\partial^2\pi},
\end{align*}
where $M$ is the stacked error matrix at the finite differences. Therefore by ensuring sufficiently small $\varepsilon$ and finite difference loss, the bound on the Jacobian error can be made arbitrarily small.

Alternatively, if the finite differences $\delta_i$ are sampled from a uniform distribution on a sphere of radius $\varepsilon$ for each evaluation of $\ell_{1,\mathsf{FD}}$ (i.e,\ the expert can be cheaply queried during training),  \citet[Theorem 3.15]{woolfe2008fast} shows that
\begin{align*}
    \|\partial_x^i \Delta_t^{\pi_\star}(\xi; \pi)\| \leq \frac{0.8\sqrt{d}}{\zeta^{1/N}}\left( \frac{1}{\varepsilon} \ell_{p,\mathsf{FD}}(\xi; \pi, \{\delta_i\}_{i=1}^{N}) + \varepsilon L_{\partial^2\pi}\right),
\end{align*}
with probability $1 - \zeta$, where $d$ is the dimensionality of the state space. This suggests that provided the expert can be requeried each iteration, $N \ll d$ finite differencing terms can be used.

\subsection{Experimental results}
\begin{figure}[H]
    \centering
    \includegraphics[width=1.0\linewidth]{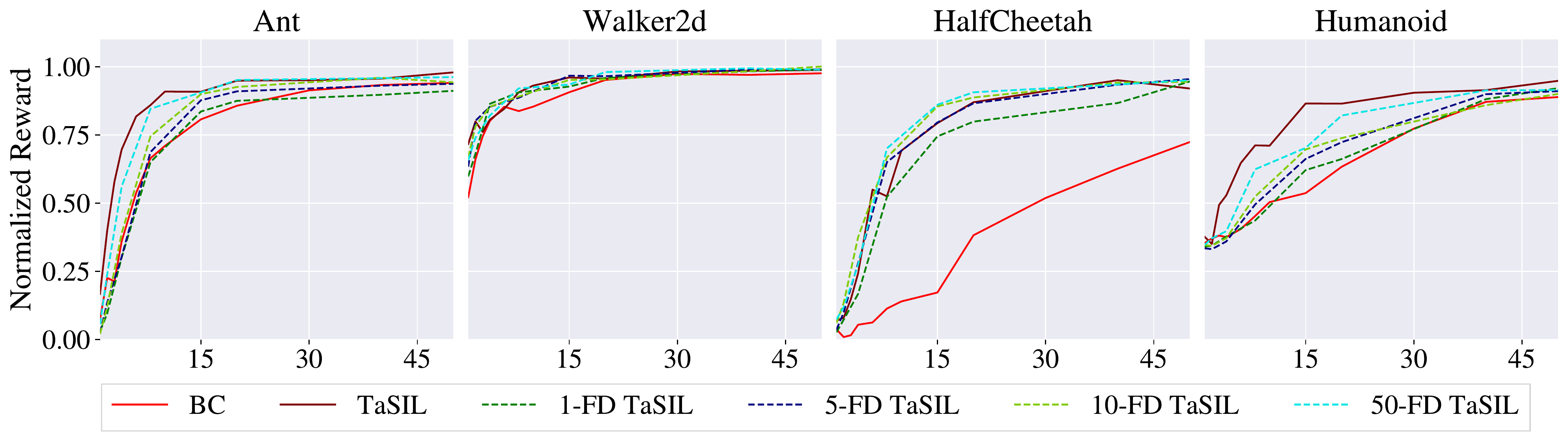}
    \caption{Mean normalized reward for vanilla Behavior Cloning, Behavior Cloning with TaSIL loss, and Behavior Cloning with finite-differencing based TaSIL. The average across 5 random seeds is shown.}
    \label{fig: finite diff}
\end{figure}

We perform several experiments using finite differencing to approximate minimizing the higher order derivatives. Figure \ref{fig: finite diff} shows configurations with $1$, $5$, $10$, and $50$ difference vectors across different MuJoCo environments. The different vectors drawn from a uniform distribution over a sphere of radius $0.01$. The difference vectors were drawn once for each state-action pair and did not change during training. This was done to simulate the effect of getting progressively closer to using the full standard basis with additional finite differencing terms.

For Walker2d and HalfCheetah with a state dimension of $17$, finite difference with a single random vector is sufficient to achieve performance on par with TaSIL using the explicit Jacobians. Humanoid and Ant with higher dimensional observation  spaces (376 and 111 dimensions respectively) also show significant improvements the more finite differences are used.

\section{Additional information for stability experiments}\label{appendix: stability experiment details}

\begin{theorem}\label{thm: diss stability}
For $\eta \in [0, 1)$, the system
\begin{align}\label{eqn: dummy system}
x_{t +1} = \eta x_t + (1 - \eta) \cdot \frac{\gamma(\|h(x_t) + u_t\|)}{\|h(x_t) + u_t\|} (h(x_t) + u_t),
\end{align}
is $\delta$-ISS around $\pi_\star(x) = - h(x)$ with class $\calK$ function $\gamma$.
\end{theorem}
\begin{proof}
We use the shorthand
$\state_{t}(\xi_1) := \state_{t}(\xi_1, \{u_k\}_{k=0}^{t-1})$
and $\state_{t}(\xi_2) := \state_{t}(\xi_2, \{0\}_{k=0}^{t-1})$.
We can prove this directly using
\begin{align*}
    &\|\state_{t+1}(\xi_1) - \state_{t+1}(\xi_2)\| \\
    &= \left\|\eta(\state_{t}(\xi_1) - \state_{t}(\xi_2)) + (1 - \eta) \cdot \frac{\gamma(\|h(x_t) + u_t\|)}{\|h(x_t) + u_t\|} (h(x_t) + u_t) \right\| \\
    &\leq \eta\|\state_{t}(\xi_1) - \state_{t}(\xi_2)\| + (1 - \eta) \gamma(\|h(x_t) + u_t\|).
\end{align*}
Since $\state_0(\xi_1) = \xi_1$ and $\state_1(\xi_2) = \xi_2$, repeated composition of this upper bound yields
\begin{align*}
    \|\state_{t}(\xi_1) - \state_{t}(\xi_2)\| &\leq \eta^t\|\xi_1 - \xi_2\| + \sum_{k=0}^{t-1}\eta^{t-1-k}(1 - \eta) \gamma(\|h(x_k) + u_k\|) \\
    &\leq \eta^t\|\xi_1 - \xi_2\| + \max_{0\leq k \leq t - 1} \gamma(\|h(x_k) + u_k\|) \\
    &= \eta^t\|\xi_1 - \xi_2\| + \gamma\paren{ \max_{0\leq k \leq t-1} \|h(x_k) + u_k\|}.
\end{align*}
\end{proof}

\paragraph{Experiment details} 
The expert MLP has two hidden layers of 32 units each with GELU activations while the learned policy has three hidden layers of 64 units and GELU activations. A tanh nonlinearity was applied to obtain the final policy output. Expert weights were initialized using Lecun Normal initialization \citet{lecun2012efficient} for the kernels and drawn form a normal distribution with $\Sigma = 0.1I$ for the biases. The learned policy weights are initialized using orthogonal intialization for the kernels and zeros for the bias.

For all stability experiments we train on 20 trajectories of length $T = 100$. Initial states were sampled from a standard normal distribution. The state-action pairs are shuffled independently into batches of size $100$ and weight updates were performed using the Adam optimizer with $\beta_1 = 0.9, \beta_2 = 0.999$, and $\varepsilon = 1 \times 10^{-4}$. The training rate was decayed with a cosine learning rate decay using an initial rate of $\alpha = 1 \times 10^{-3}$. We additionally employed $\ell^2$ weight regularization with $\lambda = 0.01$. All training is run for $4500$ iterations on our internal cluster.

To weight the various derivative terms for the different TaSIL losses we use $\lambda_0 = 1, \lambda_1 = 1$, and $\lambda_2 = 10$.

\section{Additional information for MuJoCo experiments}\label{appendix: mujoco experiment details}

We use a $\beta$-decay-rate of $p = 0.5$ for DAgger and $\alpha = T \mathrm{Tr}[\Sigma_k]$ for DART, the same parameters used by \citet{laskey2017dart} for their Mujoco experiments. For DART, we use an independent sample of $5$ trajectories to update the noise statistics. The same optimization setup from the stability experiments was used, with a batch size of $100$, Adam optimizer with $\beta_1 = 0.9$, $\beta_2 = 0.999$, $\varepsilon = 1 \times 10^{-4}$, cosine learning rate scheduling with an initial learning rate of $1 \times 10^{-3}$ decaying over the entire training duration of 4500 epochs, and $\ell^2$ weight regularization with $\lambda = 0.01$.

We train over 4500 epochs for all experiments with a training and test trajectory length of $T = 300$. All TaSIL losses use $\lambda_0 = 1$. $\lambda_1 = 0.01$ is used for the jacobian term in the $1$-TaSIL loss.

Similar to \citet{laskey2017dart}, DAgger rollout policies and DART noise statistics were updated sparsely rather than after every trajectory. We performed updates after $1, 5, 20, $ and $30$ trajectories.

\end{document}